\documentclass[11pt]{article} 

\usepackage[bookmarks,colorlinks,breaklinks]{hyperref}  
\hypersetup{linkcolor=blue,citecolor=blue,filecolor=blue,urlcolor=blue} 
\usepackage{fullpage}
\usepackage{amsmath,amsfonts,amsthm,amssymb,xspace,bm, verbatim}
\usepackage{graphicx}
\usepackage{url,algorithm,algorithmic}

\usepackage{amsmath,amsfonts,amsthm,amssymb,xspace,bm, verbatim,dsfont}
\usepackage{graphicx}
\usepackage{url,algorithm,algorithmic}

\theoremstyle{plain}
\newtheorem{theorem}{Theorem}[section]
\newtheorem{lemma}[theorem]{Lemma}
\newtheorem{corollary}[theorem]{Corollary}
\theoremstyle{definition}

\newcommand{\iprod}[2]{\langle #1, #2 \rangle}   

\newcommand{\rn}{\mathbb{R}^n}

\newcommand{\reals}{\mathbb{R}}

\newcommand{\cn}{\mathbb{C}^n}

\newcommand{\complex}{\mathbb{C}}

\newcommand{\vecfont}[1]{\mathbf{#1}}
\newcommand{\mat}[1]{\mathbf{#1}}
\newcommand{\abs}[1]{\left|#1\right|}
\newcommand{\phase}[1]{\mathrm{Ph}\left(#1\right)}

\newcommand{\conj}[1]{\overline{#1}}
\newcommand{\imag}[1]{\textrm{Im}\left(#1\right)}

\newcommand{\e}{\vecfont{e}}

\newcommand{\note}[1]{\marginpar{\tiny *note in TeX*}}
\newcommand{\ignore}[1]{}

\renewcommand{\phi}{\varphi}

\newcommand{\normal}{\mathcal{N}}

\newcommand{\eqdef}{\stackrel{\textrm{def}}{=}}

\newcommand{\dist}[2]{\|#1-#2\|_2}
\newcommand{\distop}[2]{\mathrm{dist}\left(#1,#2\right)}

\newcommand{\twonorm}[1]{\left\| {#1} \right\|_2}

\DeclareMathOperator*{\argmin}{argmin}
\DeclareMathOperator*{\argmax}{argmax}

\newcommand{\E}{\mathbb{E}}
\newcommand{\expec}[1]{\mathbb{E}\left[#1\right]}
\newcommand{\prob}[1]{\mathbb{P}\left[#1\right]}

\newcommand{\ip}[2]{\langle #1, #2 \rangle}

\newcommand{\order}[1]{O\left({#1}\right)}

\newcommand{\x}{\vecfont{x}}
\newcommand{\xo}{\vecfont{x^*}}
\newcommand{\y}{\vecfont{y}}
\renewcommand{\a}{\vecfont{a}}
\newcommand{\z}{\vecfont{z}}
\newcommand{\zhat}{\vecfont{\widehat{z}}}

\newcommand{\xplus}{\x^+}
\newcommand{\rededits}[1]{{\color{black} {#1}}}
\newcommand{\red}[1]{{\color{black} {#1}}}
\begin{document}
\renewcommand{\baselinestretch}{1}
\title{ Phase Retrieval using Alternating Minimization
\thanks{Copyright (c) 2015 IEEE. Personal use of this material is permitted. However, permission to use this material for any other purposes must be obtained from the IEEE by sending a request to pubs-permissions@ieee.org.}
}

\author{
%
%
Praneeth Netrapalli
\thanks{Microoft Research New England, Cambridge MA 02142 USA.
Email:praneeth@microsoft.com}
\and
Prateek Jain
\thanks{Microsoft Research India, Bangalore, India.
Email:prajain@microsoft.com}
\and
Sujay Sanghavi
\thanks{The University of Texas at Austin, Austin TX 78712 USA.
Email:sanghavi@mail.utexas.edu}
}

\maketitle
\begin{abstract}
Phase retrieval problems involve solving linear equations, but with missing  sign (or phase, for complex numbers) information.
More than four decades after it was first proposed, the seminal error reduction algorithm of Gerchberg and Saxton \cite{GerchbergS72}
and Fienup \cite{Fienup1982} is still the popular choice for solving many variants of this problem. The algorithm is based on
alternating minimization; i.e. it alternates between estimating the missing phase information, and the candidate solution.
Despite its wide usage in practice, no global convergence guarantees for this algorithm are known. In this paper, we show
that a (resampling) variant of this approach converges geometrically to the solution of one such problem  -- finding a
vector $\vecfont{x}$ from $\vecfont{y},\mat{A}$, where $\vecfont{y} = |\mat{A}^T\vecfont{x}|$ and $|\vecfont{z}|$
denotes a vector of element-wise magnitudes of $\vecfont{z}$ -- under the assumption that $\mat{A}$ is Gaussian.

Empirically, we demonstrate that alternating minimization performs similar to recently proposed convex techniques for this
problem (which are based on ``lifting" to a convex matrix problem) in sample complexity and robustness to noise.
However, it is much more efficient and can scale to large problems.
Analytically, for a resampling version of alternating minimization, we show geometric convergence to the solution, and sample complexity that is off by log factors from obvious lower bounds. We also establish close to optimal scaling for the case when the unknown vector is sparse. Our work represents the first theoretical guarantee for alternating minimization (albeit with resampling) for any variant of phase retrieval problems in the non-convex setting.
\end{abstract}

\section{Introduction}
\label{sec:intro}
In this paper we are interested in recovering a complex
vector $\vecfont{x^*}\in \mathbb{C}^n$ from {\em magnitudes of} its linear measurements. That is, 
for $\vecfont{a}_i \in \mathbb{C}^n$, if
\begin{align} \label{eqn:magnitude-measurements}
y_i ~ = ~ | \langle \vecfont{a_i} , \vecfont{x^*} \rangle |, \quad \text{for $i = 1,\ldots,m$}
\end{align}
then the task is to recover $\vecfont{x^*}$ using $\y$ and the measurement matrix $\mat{A} = [\vecfont{a_1}\ \vecfont{a_2}\ \dots\ \vecfont{a_m}]$. 

The above problem arises in  many settings where it is harder / infeasible to record the phase of measurements, while recording the magnitudes is significantly easier. 
This problem, known as {\em phase retrieval}, is encountered in several applications in  crystallography, optics, spectroscopy and tomography \cite{Millane1990,Hurt01}.
 Moreover, the problem is broadly studied in the following two settings:
\begin{itemize}
 \item [(i)] The measurements in \eqref{eqn:magnitude-measurements} correspond to the Fourier transform (the number of measurements here is equal to $n$)
and there is some apriori information about the signal.
 \item [(ii)] The set of measurements $\y$ are overcomplete (i.e., $m > n$), while some apriori information about the signal may or may not be available. 
\end{itemize}
In the first case, various types of apriori information about the
underlying signal such as positivity, magnitude information on the signal \cite{Fienup1982}, sparsity \cite{ShechtmanESS11} and so on have been studied.
In the second case, algorithms for various measurement schemes such as Fourier oversampling \cite{Misell73}, multiple random illuminations \cite{CandesESV13,WaldspurgerdAM12} and wavelet transform \cite{ChiRS05}
have been suggested.

By and large, the most well known methods for solving this problem are the error reduction algorithms due to Gerchberg and Saxton \cite{GerchbergS72}
and Fienup \cite{Fienup1982}, and variants thereof. These algorithms are alternating projection algorithms that iterate between the unknown phases of the measurements and the unknown underlying vector.
Though the empirical performance of these algorithms has been well studied \cite{Fienup1982,Marchesini07,Marchesini07b}.
and they are used in many applications \cite{MiaoCKS99,Miaoetal02},
there are not many theoretical guarantees regarding their performance.


More recently, a  line of work \cite{ChaiMP11,CandesSV12,WaldspurgerdAM12} has approached this problem from a different angle, based on the realization that recovering $\vecfont{x^*}$ is equivalent to recovering the
rank-one matrix $\vecfont{x^*}\vecfont{x^*}^T$, i.e., its outer product. Inspired by the recent literature on trace norm relaxation of the rank constraint, they design SDPs to solve this problem.
Refer Section \ref{sec:relwork} for more details.

{In this paper} we go back to the empirically more popular ideology of alternating minimization; we develop a new
alternating minimization algorithm, and show that {\em (a)} empirically, it noticeably outperforms convex
methods, and {\em (b)} analytically, a natural resampled version of this algorithm requires $O(n\log^3 n \log \frac{1}{\epsilon})$ i.i.d.
random Gaussian measurements to geometrically converge to the true vector up to an accuracy of $\epsilon$.\\
{\bf Our contribution}:
\begin{itemize}
  \item	{The iterative part of our algorithm is essentially due to Gerchberg and Saxton \cite{GerchbergS72} and Fienup \cite{Fienup1982}; indeed, with out resampling, our algorithm is exactly their famous error reduction algorithm;} the novelty in our \emph{algorithmic contribution} is
the initialization step which makes it more likely for the iterative procedure to succeed - see Figures~\ref{fig:sense}, \ref{fig:randomgaussianfilters} and \ref{fig:noisy-initialization}. 
  \item	{Our \emph{analytical contribution} is the first theoretical guarantee establishing the correctness of alternating minimization (with resampling) in recovering the underlying signal for the phase retrieval problem.}
  \item	When the underlying vector is \emph{sparse}, we design another algorithm that achieves a sample complexity of $\order{\left(x^*_{\textrm{min}}\right)^{-4} \log n + k \left(\log^3 k + \log \frac{1}{\epsilon} \log \log \frac{1}{\epsilon}\right)}$
and computational complexity of $O\left( \left(x^*_{\textrm{min}}\right)^{-4} kn \log n + k^2 \log^2 \frac{1}{\epsilon} \log \log \frac{1}{\epsilon} \right)$,
where $k$ is the sparsity and $x^*_{\textrm{min}}$ is
the minimum non-zero entry of $\vecfont{x^*}$. This algorithm also runs over $\cn$ and scales much better than SDP based methods.
\end{itemize}

Besides being an empirically better algorithm for this problem, our work is also interesting in a broader sense: there are many problems in machine learning,
signal procesing and numerical linear algebra, where the natural formulation of a
problem is non-convex; examples include rank constrained problems, applications of EM algorithms etc., and alternating minimization has good empirical performance.
However, the methods with the best (or only) analytical guarantees involve convex relaxations (e.g., by relaxing the rank constraint and penalizing the trace norm).
In most of these settings, correctness of alternating minimization is an open question.
We believe that our results in this paper are of interest, and may have implications, in this larger context.

\textbf{Difference from standard alternating minimization}:
The algorithm we analyze in this paper uses different measurements in each iteration
and differs from standard alternating minimization approaches in this context, where same measurements are used
in each iteration. Since our algorithm decays the error at a geometric rate, an error of $\epsilon$ requires
$\order{\log(1/\epsilon)}$ iterations, increasing the total number of measurements by this factor. Theoretically, this is still competitive with convex optimization approaches under computational constraints.
Indeed, for a $\textrm{poly}(n)$ run time, the best known bounds for phase retrieval via convex optimization can guarantee an accuracy of $1/\textrm{poly}(n)$. For an accuracy of $\epsilon = 1/\textrm{poly}(n)$, the use of different samples in different iterations of our
algorithm contributes an extra factor of just $\order{\log n}$.
{
Nevertheless, throwing away samples (as our algorithm does) is simply not a viable option in many practical settings. In fact, we empirically observe that using the same samples in all iterations performs significantly better than using different samples in each iteration (indeed, for our numerical experiments, we use the same samples in each iteration). Subsequent to our work, Cand{\`e}s et al. \cite{CandesLS14} proposed a non-convex iterative algorithm based on Wirtinger flow, that uses same samples in each iteration, and show that it converges to the true underlying vector. See Section~\ref{sec:relwork} for more details.
}
The rest of the paper is organized as follows: In section \ref{sec:relwork}, we briefly review related work. We clarify our notation in Section \ref{sec:notation}.
We present our algorithm in Section \ref{sec:algo} and the main results in Section \ref{sec:sense_analysis}. We present our results for the sparse case in Section \ref{sec:sparse}.
Finally, we present  experimental results in Section \ref{sec:experiments}.


\subsection{Related Work}\label{sec:relwork}
{\bf Phase Retrieval via Non-Convex Procedures}:
Inspite of the huge amount of work it has attracted, phase retrieval has been a long standing open problem.
Early work in this area focused on using holography to capture the phase information along with magnitude measurements \cite{Gabor48,LeithU62}.
However, computational methods for reconstruction of the signal using only magnitude measurements received a lot of attention
due to their applicability in resolving spurious noise, fringes,
optical system aberrations and so on and difficulties in the implementation of interferometer setups \cite{Duadietal11}.
Though such methods have been developed to solve this problem in various practical settings \cite{DaintyF87,FienupMSS93,MiaoCKS99,Miaoetal02},
our theoretical understanding of this problem is still far from complete. Many papers \cite{BruckS79,Hayes82,Sanz85}
have focused on determining conditions under which \eqref{eqn:magnitude-measurements} has a unique solution.
However, the uniqueness results of these papers
do not resolve the algorithmic question of how to find the solution to \eqref{eqn:magnitude-measurements}.

Since the seminal work of Gerchberg and Saxton \cite{GerchbergS72} and Fienup \cite{Fienup1982}, many iterated projection algorithms have been
developed targeted towards various applications \cite{AbrahamsL96,Elser03,BauschkeCL03}. \cite{Misell73} first suggested the use of multiple
magnitude measurements to resolve the phase problem. This approach has been successfully used in many practical applications - see \cite{Duadietal11}
and references there in.
Following the empirical success of these algorithms, researchers were able to
explain its success in some of the instances \cite{YoulaW82,TrussellC84}
using Bregman's theory of iterated projections onto convex sets \cite{Bregman65}.
However, many instances, such as the one we consider in this paper, are out of reach of this theory since they involve magnitude constraints which are non-convex.
To the best of our knowledge, there are no theoretical results on the convergence of these approaches in a non-convex setting.

Subsequent to our work, Cand{\`e}s et al. \cite{CandesLS14} proposed an iterative algorithm based on Wirtinger flow
which is similar to optimizing a non-convex function using gradient descent. Despite using same samples, they manage
to show that their algorithm recovers the true underlying vector for Gaussian measurements, albeit with a slow convergence rate.
\rededits{
Quite interestingly, they also show that if the initial point is $\order{\frac{1}{\sqrt{n}}}$ close to the true vector
(which can be achieved by using a small amount of resampling), their algorithm (using same samples) achieves exact recovery for Gaussian measurements as well as {\em coded diffraction} measurements
(which are practically more relevant than Gaussian measurements), with a fast convergence rate matching that of our algorithm.}
It has also been reported that the Wirtinger flow algorithm has better properties than alternating minimization in some
optics settings \cite{BianSZGCD2014}.

{\bf Phase Retrieval via Convex Relaxation}:
An interesting recent approach for solving this problem formulates it as one of finding the rank-one solution to a system of linear matrix equations. The papers
\cite{ChaiMP11,CandesSV12} then take the approach of relaxing the rank constraint by a trace norm penalty, making the overall algorithm a convex program (called \emph{PhaseLift}) over $n\times n$ matrices.
Another recent line of work \cite{WaldspurgerdAM12} takes a similar but different approach : it uses an SDP relaxation (called \emph{PhaseCut}) that is inspired by the classical SDP relaxation for the max-cut problem.
To date, these convex methods are the only ones with analytical guarantees on statistical performance
(i.e. the number $m$ of measurements required to recover $\vecfont{x^*}$) \cite{CandesL12,WaldspurgerdAM12}.
However, by ``lifting" a vector problem to a matrix one, these methods lead to a much larger representation of the state space,
and higher computational cost as a result.

{\bf Measurement Schemes}:
Earlier results on PhaseLift and PhaseCut \cite{CandesL12,WaldspurgerdAM12} assumed an i.i.d. random Gaussian model on
the measurement vectors $\vecfont{a}_i$.
\cite{GrossKK13} extends these results for
PhaseLift for measurement schemes known as t-designs, which are more general than Gaussian measurements.
Recently, \cite{CandesLS13} establishes near-optimal statistical guarantees for PhaseLift under masked Fourier transform
measurements.

{\bf Sparse Phase Retrieval}:
A special case of the phase retrieval problem which has received a lot of attention recently is when the underlying signal $\vecfont{x^*}$ is known to be sparse.
Though this problem is closely related to the compressed sensing problem, lack of phase information makes this harder. However, the $\ell_1$ regularization
approach of compressed sensing has been successfully used in this setting as well.
In particular, if $\vecfont{x^*}$ is sparse, then the corresponding lifted matrix $\vecfont{x^*}\vecfont{x^*}^T$ is also sparse.
\cite{ShechtmanESS11,OhlssonYDS11,LiV12}
use this observation to design $\ell_1$ regularized SDP algorithms for phase retrieval of sparse vectors. For random Gaussian measurements,
\cite{LiV12} shows that $\ell_1$ regularized PhaseLift recovers $\vecfont{x^*}$ correctly
if the number of measurements is $\Omega(k^2 \log n)$. By the results of \cite{OymakJFEH12}, this result is tight up to logarithmic factors for $\ell_1$ and trace norm
regularized SDP relaxations.
\cite{JaganathanOH12,ShechtmanBE13} develop algorithms for phase retrieval from Fourier magnitude measurements.
However, achieving the optimal sample complexity of $\order{k \log \frac{n}{k}}$ is still open \cite{EldarM12}.

{\bf Alternating Minimization} (a.k.a. \textbf{ALS}):
Alternating minimization has been successfully applied to many applications in the low-rank matrix setting.
For example, clustering \cite{KimPark08b}, sparse PCA \cite{ZouHT06}, non-negative matrix factorization \cite{KimPar08},
signed network prediction \cite{HsiehCD12} etc.
However, despite empirical success, for most of the problems, there are no theoretical guarantees regarding its convergence except to a local minimum.
Of late, however, there has been a spurt of work in obtaining provable guarantees for alternating minimization in various
settings such as learning sparsely used dictionaries \cite{AgarwalANJ2013}, matrix completion \cite{JainN2014}, robust
PCA \cite{NetrapalliNSAJ2014} etc. \rededits{Though earlier results for matrix completion \cite{Keshavan12,JainNS12,Hardt13} use
heavy resampling, subsequent work \cite{JainN2014} has obtained similar results with a small amount of resampling.}

There has also been some work on designing other non convex optimization algorithms, such as gradient descent for solving some of these problems. For instance, \cite{KeshavanOM2009,KeshavanOM2009Noisy} propose a gradient descent algorithm on the Grassmanian manifold to solve the matrix completion problem.
\section{Notation}\label{sec:notation}\vspace*{-5pt}
We use bold capital letters ($\mat{A},\mat{B}$ etc.) for matrices, bold small case letters ($\vecfont{x},\vecfont{y}$ etc.) for vectors
and non-bold letters ($\alpha, U$ etc.) for scalars.
For every complex vector $\vecfont{w} \in \mathbb{C}^n$, $|\vecfont{w}|\in \mathbb{R}^{n}$ denotes its element-wise magnitude vector.
$\vecfont{w}^{T}$ and $\mat{A}^{T}$ denote the Hermitian transpose of the vector $\vecfont{w}$ and the matrix $\mat{A}$ respectively.
$\vecfont{e_1},\vecfont{e_2},$ etc. denote the canonical basis vectors in $\cn$.
$\overline{z}$ denotes the complex conjugate of the complex number $z$.
In this paper we use the standard Gaussian (or normal) distribution over $\cn$. $\vecfont{a}$ is said to be distributed according to this
distribution if $\vecfont{a} = \vecfont{a_1} + i \vecfont{a_2}$, where $\vecfont{a_1}$ and $\vecfont{a_2}$ are independent and are distributed
according to $\mathcal{N}\left(0,\mat{I}\right)$.
We also define $\phase{z} \eqdef \frac{z}{|z|}$ for every $z \in \complex$, and 
$\distop{\vecfont{w_1}}{\vecfont{w_2}} \eqdef \sqrt{1 - \left|\frac{\iprod{\vecfont{w_1}}{\vecfont{w_2}}}{\twonorm{\vecfont{w_1}}\twonorm{\vecfont{w_2}}}\right|^2}$ for every $\vecfont{w_1},\vecfont{w_2} \in \cn$.
Finally, we use the shorthand wlog for without loss of generality and whp for with high probability.\vspace*{-5pt}
\section{Algorithm}\label{sec:algo}\vspace*{-5pt}
In this section, we present our alternating minimization based algorithm for solving the phase retrieval problem. Let $\mat{A} \in \mathbb{C}^{n\times m}$ be the measurement matrix, with $\vecfont{a}_i$ as its $i^{th}$ column; similarly let $\vecfont{y}$ be the vector of recorded magnitudes. Then,
\[
\vecfont{y} ~ = ~ | \, \mat{A}^T \vecfont{x^*} \, |.
\]
Recall that,  given $\vecfont{y}$ and $\mat{A}$, the goal is to recover $\vecfont{x}^*$. 
If we had access to the true phase $\vecfont{c^*}$ of $A^T\vecfont{x^*}$ (i.e., $c^*_i=\phase{\ip{\vecfont{a_i}}{\xo}}$) and $m\geq n$, then our problem reduces to one of solving a system of linear equations:
\[
\mat{C}^* \vecfont{y}  ~ = ~ \mat{A}^T \vecfont{x^*},
\]
where $\mat{C}^* \eqdef  \mbox{Diag}(\vecfont{c}^*)$ is the diagonal matrix of phases.
Of course we do not know $\mat{C}^*$, hence one approach to recovering $\xo$ is to solve:
\begin{equation}
\argmin_{\mat{C},\vecfont{x}} ~ \|\mat{A}^T \vecfont{x} - \mat{C}\vecfont{y}\|_2, \label{eq:prob}\end{equation}
where $\vecfont{x} \in \cn$ and $\mat{C} \in \complex^{m \times m}$ is a diagonal matrix with each diagonal entry of magnitude $1$.
Note that the above problem is {\em not convex} since $\mat{C}$ is restricted to be a diagonal phase matrix and hence, one cannot use standard convex optimization methods to solve it.

Instead, our algorithm uses the well-known alternating minimization: alternatingly update $\vecfont{x}$ and $\mat{C}$ so as to minimize \eqref{eq:prob}.
Note that given $\mat{C}$, the vector $\vecfont{x}$ can be obtained by solving the following least squares problem: $\min_{\vecfont{x}} \|\mat{A}^T\vecfont{x}-\mat{C}\y\|_2$.
Since the number of measurements $m$ is larger than the dimensionality $n$ and since each entry of $\mat{A}$ is sampled from independent Gaussians, $\mat{A}$ is invertible with probability $1$. Hence, the above least squares problem has a unique solution. On the other hand, given $\vecfont{x}$, the optimal $\mat{C}$ is given by $\mat{C}=\mbox{Diag}\left(\phase{\mat{A}^T\vecfont{x}}\right)$. 

\begin{algorithm}[t]
\caption{AltMinPhase}
\label{algo:phasesensing-nonsparse-no-resample}
\begin{algorithmic}[1]
\INPUT $\mat{A},\vecfont{y},t_0$
\STATE Initialize $\vecfont{x^0} \leftarrow \mbox{top singular vector of }\sum_i y_i^2 \vecfont{a_i} \vecfont{a_i}^T$
\FOR{$t = 0,\cdots,t_0-1$}
\STATE	$\mat{C^{t+1}} \leftarrow \mbox{Diag}\left(\phase{\mat{A}^T \vecfont{x^{t}}}\right)$
\STATE	$\vecfont{x^{t+1}} \leftarrow \argmin_{\vecfont{x}\in \rn} \twonorm{\mat{A}^T \vecfont{x} - \mat{C^{t+1}}\vecfont{y}}$
\ENDFOR
\OUTPUT $\vecfont{x^{t_0}}$
\end{algorithmic}
\end{algorithm}

While the above algorithm is simple and intuitive, it is known that with bad initial points, the solution might not converge to $\xo$.
In fact, this algorithm with a uniformly random initial point has been empirically evaluated for example in \cite{WaldspurgerdAM12}, where
it performs worse than SDP based methods.
Moreover, since the underlying problem is non-convex, standard analysis techniques fail to guarantee convergence to the global optimum, $\xo$. Hence, the key challenges here are: a) a good initialization step for this method, b) establishing this method's convergence to $\xo$. 

We address the first key challenge in our AltMinPhase algorithm (Algorithm~\ref{algo:phasesensing-nonsparse-no-resample}) by initializing $\vecfont{x}$ as the
largest singular vector of the matrix $\mat{S}=\frac{1}{m}\sum_i y_i^2 \vecfont{a_i} \vecfont{a_i}^T$. {This is similar to the initialization in \cite{KeshavanOM2009} for the matrix completion problem.} Theorem~\ref{thm:firststep-svd} shows that when $\mat{A}$ is sampled from
standard complex normal distribution, this initialization is accurate. In particular, if $m\geq C_1n \log^3 n$ for large enough $C_1>0$, then whp we have $\|\vecfont{x^0}-\xo\|_2\leq 1/100$ (or any other constant).

Theorem~\ref{thm:convergence} addresses the second key challenge and shows that a variant of AltMinPhase (see Algorithm~\ref{algo:phasesensing-nonsparse}) actually converges to the global optimum $\xo$ at linear rate. See section~\ref{sec:sense_analysis} for a detailed analysis of our algorithm. 

We would like to stress that not only does a natural variant of our proposed algorithm have rigorous theoretical guarantees, it also is effective practically as each of its iterations is fast,
has a closed form solution and does not require SVD computation.
AltMinPhase has similar statistical complexity to that of PhaseLift and PhaseCut while being much more efficient computationally.
In particular, for accuracy $\epsilon$, we only need to solve each least squares problem only up to accuracy $\order{\epsilon^2}$.
Since the measurement matrix $A$ is Gaussian with $m> Cn$, it is well conditioned. This means that each such step takes
$\order{mn \log \frac{1}{\epsilon}}$ time using the conjugate gradient method. When $m = \order{n}$ and we have geometric convergence, the total time taken by the algorithm is $\order{n^2 \log^2 \frac{1}{\epsilon}}$. SDP based methods
on the other hand require $\Omega(n^3/\sqrt{\epsilon})$ time.
Moreover, our initialization step increases the likelihood of successful recovery as opposed to a random initialization (which has been considered so far in prior work).
Refer Figure~\ref{fig:sense} for an empirical validation of these claims.

\begin{figure*}[!ht]
  \centering
  \begin{tabular}[t]{cc}
    \includegraphics[width=.5\textwidth]{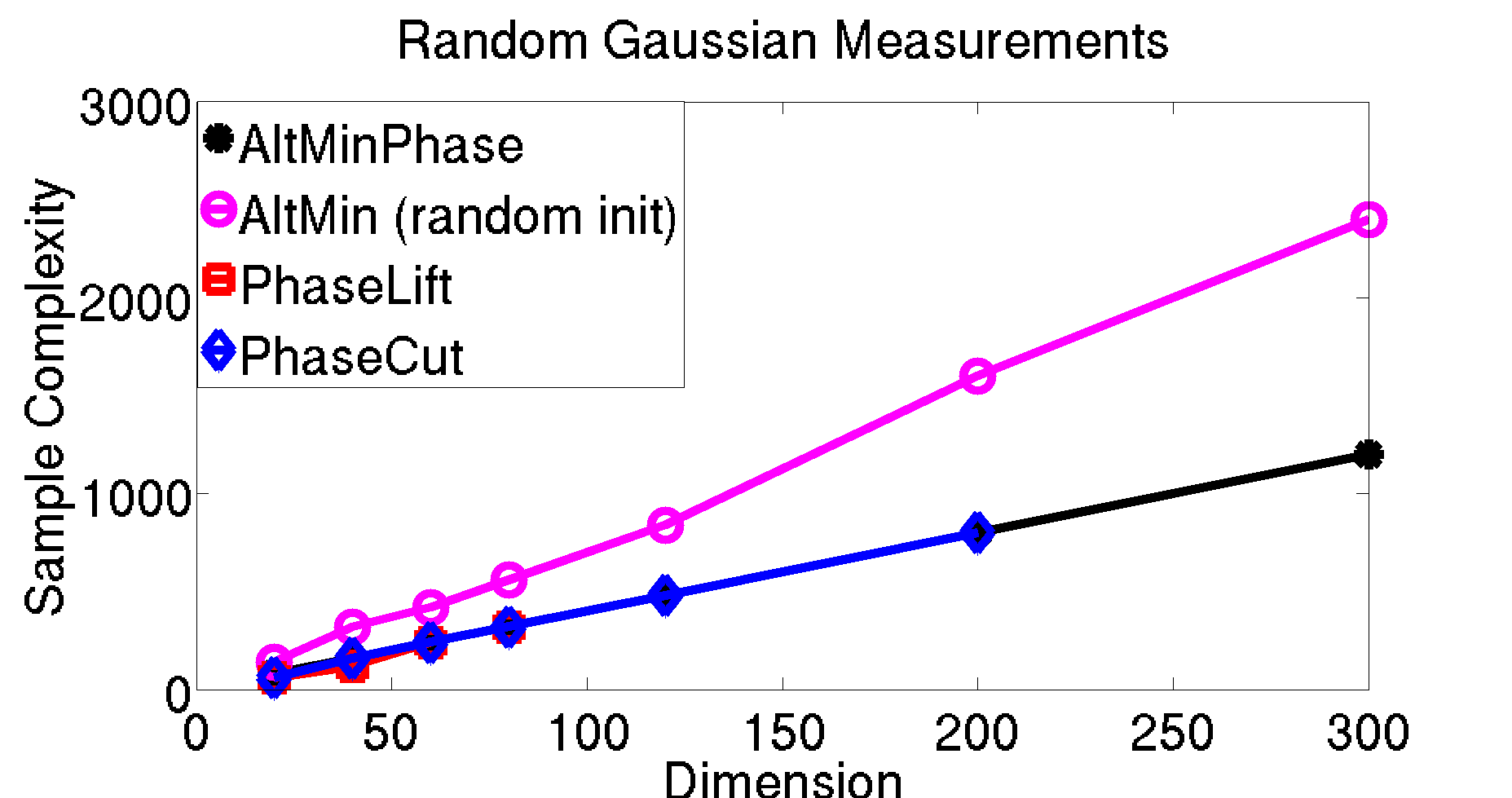}&\includegraphics[width=.5\textwidth]{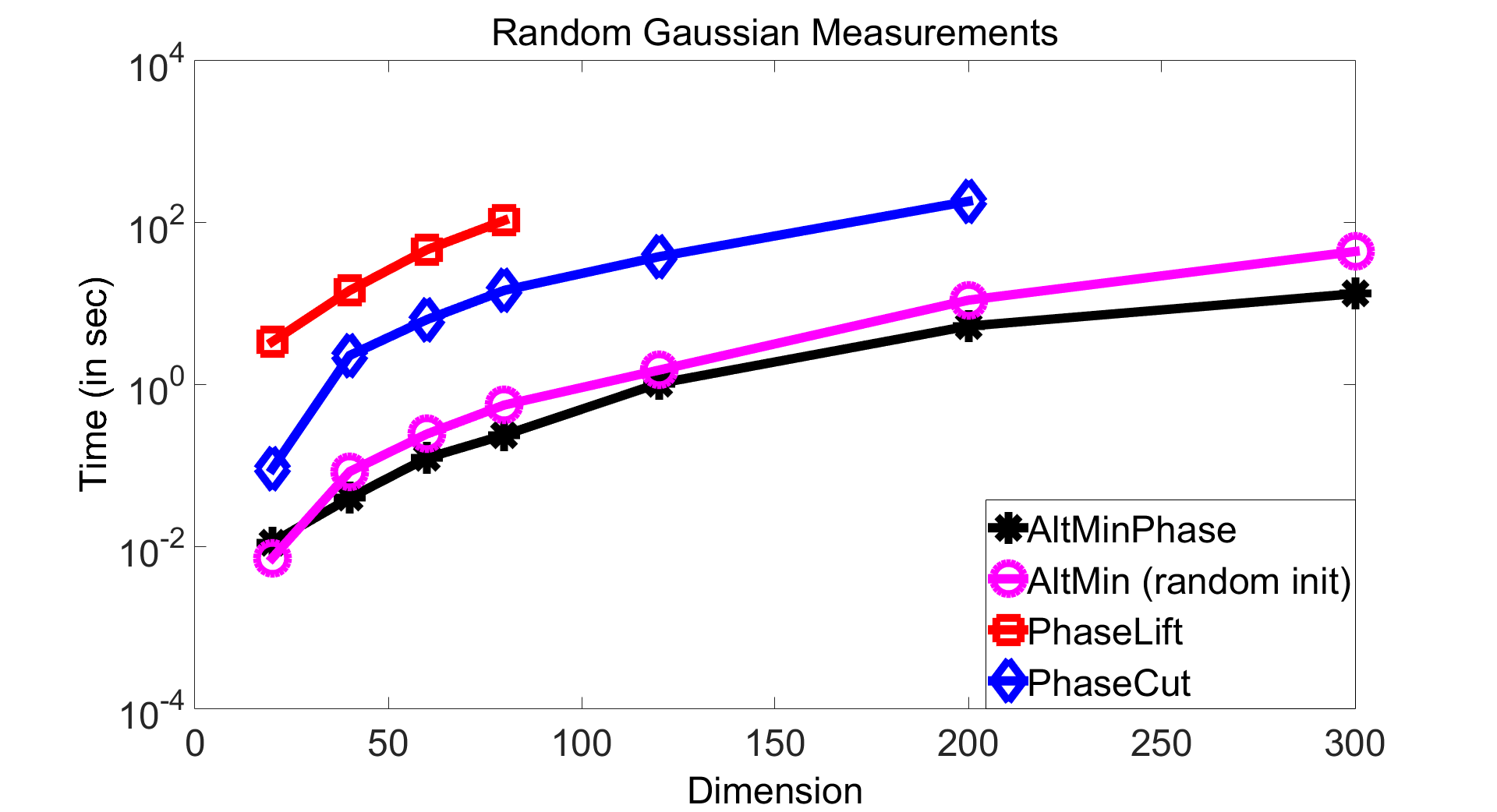}\\
{\bf (a)}&{\bf (b)}
  \end{tabular}
  \caption{Sample and Time complexity of various methods for Gaussian measurement matrices $A$.  
  Figure~\ref{fig:sense}(a) compares the number of measurements required for successful recovery by various methods.
  We note that our initialization improves sample complexity over that of random initialization (AltMin (random init)) by a factor of $2$. AltMinPhase requires similar number of measurements as PhaseLift and
  PhaseCut. 
  Figure~\ref{fig:sense}(b) compares the running time of various algorithms on log-scale. Note that AltMinPhase is almost two orders of magnitude faster than PhaseLift and PhaseCut. 
  }
  \label{fig:sense}
\end{figure*}

{
A key drawback of our results, however, is the use of resampling. More specifically, our convergence guarantee is obtained for a variant of Algorithm~\ref{algo:phasesensing-nonsparse-no-resample} (see Algorithm~\ref{algo:phasesensing-nonsparse}), where we use different samples in each iteration. In practice, this is not feasible since in many applications, taking so many measurements may not be possible. On the other hand, the SDP approaches and a recent non-convex optimization approach do not face this issue. See Section~\ref{sec:intro} for more details on this aspect.}

\section{Main Results: Analysis}\label{sec:sense_analysis}
In this section we describe the main contribution of this paper: provable statistical guarantees for the success of alternating minimization in solving the phase recovery problem.
To this end, we consider the setting where each measurement vector $\vecfont{a_i}$ is iid and is sampled from the standard complex normal distribution. We would like to stress that all
the existing guarantees for phase recovery also use exactly the same setting \cite{CandesSV12,CandesL12,WaldspurgerdAM12}. 
Table~\ref{tab:non-sparse-comparison} presents a comparison of the theoretical guarantees of Algorithm \ref{algo:phasesensing-nonsparse} as compared to PhaseLift and PhaseCut.

\begin{table*}[!ht]
  \begin{center}
    \begin{tabular}{ | c | c | c |}
      \hline
								& Sample complexity & Comp. complexity  \\ \hline
      \red{Algorithm \ref{algo:phasesensing-nonsparse}}	& ${ \order{n \log n \left(\log^2 n + \log \frac{1}{\epsilon} \log \log \frac{1}{\epsilon}\right)}}$
					&  ${\order{n^2 \log n\left(\log^2 n + \log^2 \frac{1}{\epsilon} \log \log \frac{1}{\epsilon}\right)}}$ \\ \hline
      PhaseLift \cite{CandesL12}		& $\order{n}$
					& $\order{n^3/\epsilon^2}$ \\ \hline
      PhaseCut \cite{WaldspurgerdAM12}	& $\order{n}$
					& $\order{n^3/\sqrt{\epsilon}}$ \\ \hline
      \end{tabular}
      \caption{Comparison of Algorithm~\ref{algo:phasesensing-nonsparse} with PhaseLift and PhaseCut: Though the sample complexity of Algorithm~\ref{algo:phasesensing-nonsparse} is off by $\log$ factors from that of PhaseLift and PhaseCut, it is $\order{n}$ better than them in computational complexity. Note that, we can solve the least squares problem in each iteration approximately by using fast approximte solvers such as conjugate gradient method in time $\order{mn \log \frac{1}{\epsilon}}$, since the condition number of our measurement matrix is $\Omega(1)$ (this follows for instance from Theorem~5.31 of \cite{Vershynin10}).}
      \label{tab:non-sparse-comparison}
  \end{center}
\end{table*}

 Our proof for convergence of alternating minimization can be broken into two key results. We first show that if $m\geq C n \log^3n$, then whp the initialization step used by AltMinPhase returns
$\vecfont{x^0}$ which is at most a constant distance away from $\xo$. Furthermore, that constant can be controlled by using more samples (see Theorem~\ref{thm:firststep-svd}). 

We then show that if $\vecfont{x^t}$ is a {\em fixed} vector such that $\distop{\vecfont{x^t}}{\xo} < c$ (small enough) and $\mat{A}$ is sampled independently of $\vecfont{x^t}$ with
$m > C n$ ($C$ large enough) then whp $\vecfont{x^{t+1}}$
satisfies: $\distop{\vecfont{x^{t+1}}}{\xo} < \frac{3}{4}\distop{\vecfont{x^t}}{\xo}$ (see Theorem~\ref{thm:convergence}). Note that our analysis critically requires $\vecfont{x^t}$ to be ``fixed''
and be independent of the sample matrix $\mat{A}$. Hence, we cannot re-use the same $\mat{A}$ in each iteration; instead, we need to resample $\mat{A}$ in every iteration.
Using these results, we prove the correctness of Algorithm~\ref{algo:phasesensing-nonsparse}, which is a natural resampled version of AltMinPhase.
\begin{algorithm}[th]
\caption{AltMinPhase with Resampling}
\label{algo:phasesensing-nonsparse}
\begin{algorithmic}[1]
\INPUT $\mat{A},\vecfont{y},\epsilon$
\STATE $t_0 \leftarrow c\log \frac{1}{\epsilon}$
\STATE Partition $\vecfont{y}$ and (the corresponding columns of) $A$ into $t_0+1$ equal disjoint sets: $(\vecfont{y^0},\mat{A^0}),(\vecfont{y^1},\mat{A^1}),\cdots,(\vecfont{y^{t_0}},\mat{A^{t_0}})$
\STATE $\vecfont{x^0} \leftarrow \mbox{top singular vector of } \sum_l \left(y_l^0\right)^2 \vecfont{a_{\ell}^0} \left(\vecfont{a_{\ell}^0}\right)^T$
\FOR{$t = 0,\cdots,t_0-1$}
\STATE	$\mat{C^{t+1}} \leftarrow \mbox{Diag}\left(\phase{\left(\mat{A^{t+1}}\right)^T \vecfont{x^{t}}}\right)$
\STATE	$\vecfont{x^{t+1}} \leftarrow \argmin_{\vecfont{x}\in \rn} \twonorm{\left(\mat{A^{t+1}}\right)^T \vecfont{x} - \mat{C^{t+1}}\vecfont{y^{t+1}}}$
\ENDFOR
\OUTPUT $\vecfont{x^{t_0}}$
\end{algorithmic}
\end{algorithm}

We now present the two results mentioned above.  For our proofs, wlog, we assume that $\|\xo\|_2=1$. 

Our first result guarantees a good initial vector.
\begin{theorem}\label{thm:firststep-svd}
  There exists a constant $C_1$ such that if $m > \frac{C_1}{c^2} n \log^3 n$, then in Algorithm \ref{algo:phasesensing-nonsparse}, with probability greater than $1 - 4/m^2$ we have:
\begin{align*}
{  \distop{\vecfont{x^0}}{\vecfont{x^*}} < \sqrt{c}.}
\end{align*}
\end{theorem}
\textbf{Remark}: Note that $\distop{\cdot}{\cdot}$ is invariant with the global phase i.e., $\distop{\vecfont{x^0}}{\xo} = \distop{\vecfont{x^0}}{e^{i \phi} \xo}$, for any $\phi \in [-\pi,\pi]$.

In the second result, we prove a geometric decay in $\distop{\cdot}{\cdot}$ along with a bound on the $\ell_2$ error of our estimate.
{
Since $\xo$ is unique only up to a global phase factor and $\ell_2$ error $\left(\twonorm{\vecfont{x^{t+1}} - \xo}\right)$ depends on the global phase, we choose $\xo$ such that $\ip{\vecfont{x^t}}{\xo} \geq 0$. With this choice of global phase for $\xo$, we now state our second theorem:}
\begin{theorem}\label{thm:convergence}
{  Choose the global phase factor of $\xo$ such that $\ip{\vecfont{x^t}}{\xo} \geq 0$.} There exist constants $c$, $\widehat{c}$ and $\widetilde{c}$ such that in iteration $t$ of Algorithm \ref{algo:phasesensing-nonsparse}, if $\distop{\vecfont{x^{t}}}{\xo} < c$ and
  the number of columns of $\mat{A^t}$ is greater than $\widehat{c} n\log \frac{1}{\eta}$ then, with probability more than $1 - \eta$, we have:
\begin{align*}
\distop{\vecfont{x^{t+1}}}{\xo} &< \frac{3}{4} ~ \distop{\vecfont{x^{t}}}{\xo}, \mbox{ and } \\
\dist{\vecfont{x^{t+1}}}{\xo} &< \widetilde{c} ~ \distop{\vecfont{x^{t}}}{\xo}.
\end{align*}
\end{theorem}
\begin{proof} 
For simplicity of notation in the proof of the theorem, we will use $\mat{A}$ for $\mat{A^{t+1}}$, $\mat{C}$ for $\mat{C^{t+1}}$, $\vecfont{x}$ for $\vecfont{x^t}$, $\vecfont{\xplus}$ for $\vecfont{x^{t+1}}$, and $\vecfont{y}$ for $\vecfont{y^{t+1}}$.
Now consider the update in the $(t+1)^{\mathrm{th}}$ iteration: 
\begin{align}
  \vecfont{\xplus} &= \argmin_{\vecfont{\widetilde{x}}\in \rn} \twonorm{\mat{A}^T \vecfont{\widetilde{x}} - \mat{C}\vecfont{y}}= \left(\mat{A}\mat{A}^T\right)^{-1}\mat{A}\mat{C}\vecfont{y} \nonumber\\
  &= \left(\mat{A}\mat{A}^T\right)^{-1}\mat{A}\mat{D}\mat{A}^T\vecfont{x^*},  \label{eq:xupdate}
\end{align}
where $\mat{D}$ is diagonal with $D_{ll} \eqdef \phase{\vecfont{a_{\ell}}^T\vecfont{x}\cdot \conj{\vecfont{a_{\ell}}^T\vecfont{x^*}}}$.
Now \eqref{eq:xupdate} can be rewritten as: 
\begin{align}
  \vecfont{\xplus} &= \left(\mat{A}\mat{A}^T\right)^{-1}\mat{A}\mat{D}\mat{A}^T\vecfont{\xo} \nonumber\\
  &= \xo + \left(\mat{A}\mat{A}^T\right)^{-1}\mat{A}\left(\mat{D}-\mat{I}\right)\mat{A}^T\xo,\label{eq:t2}
\end{align}
that is, $\vecfont{\xplus}$ can be viewed as a perturbation of $\xo$ and the goal is to bound the error term (the second term above). We break the proof into two main steps:
\begin{enumerate}
\item \rededits{$\exists$ a constant $c_1$ such that $\twonorm{\xo-\xplus} \leq c_1 \distop{\vecfont{x}}{\xo}$ (see Lemma~\ref{lem:xo_comp}), and}
\item $|\iprod{\vecfont{z}}{\xplus}| \leq \frac{5}{9} \distop{\vecfont{x}}{\xo}$, for all $\z$ s.t. $\z^T \xo=0$.  (see Lemma~\ref{lem:zbound})
\end{enumerate}
\rededits{
Firstly, the bound on $\twonorm{\xo-\xplus}$, by triangle inequality, implies that $\twonorm{\xplus} \geq 1- c_1 \distop{\vecfont{x}}{\xo}$.
Further it implies the following bound on $|\iprod{\xo}{\xplus}|$:
\begin{align*}
& \twonorm{\xo-\xplus}^2 \leq c_1^2 \distop{\vecfont{x}}{\xo}^2 \\
&\Rightarrow 1 + \twonorm{\xplus}^2 - 2 \iprod{\xo}{\xplus} \leq c_1^2 \distop{\vecfont{x}}{\xo}^2 \\
&\Rightarrow \iprod{\xo}{\xplus} \geq 1 - c_1 \distop{\vecfont{x}}{\xo}.
\end{align*}
}

Using the above bounds and choosing $c < \frac{1}{100c_1}$, we can prove the theorem:
\begin{align*}
  \distop{\vecfont{\xplus}}{\xo}^2 
&= \frac{\max_{\z \perp \xo} \abs{\iprod{\z}{\xplus}}^2}{\abs{\iprod{\xo}{\xplus}}^2 + \max_{\z \perp \xo} \abs{\iprod{\z}{\xplus}}^2}\\
 < & \frac{(25/81)\cdot \distop{\vecfont{x}}{\xo}^2}{(1-c_1\distop{\vecfont{x}}{\xo})^2}
  \leq \frac{9}{16}\distop{\vecfont{x}}{\xo}^2,
\end{align*}
proving the first part of the theorem. The second part follows easily from \eqref{eq:t2} and Lemma \ref{lem:xo_comp}.
\end{proof}

{\bf Intuition and key challenge}:
If we look at step 6 of Algorithm \ref{algo:phasesensing-nonsparse}, we see that, for the measurements, we use magnitudes calculated from $\xo$ and phases calculated from $\vecfont{x}$.
Intuitively, this means that we are trying to push $\vecfont{\xplus}$ towards $\xo$ (since we use its magnitudes) and $\vecfont{x}$ (since we use its phases)
at the same time. The key intuition behind the success of this procedure is that the push towards $\xo$ is stronger than the push towards $\vecfont{x}$, when
$\vecfont{x}$ is close to $\xo$.
The key lemma that captures this effect is stated below:
\begin{lemma}\label{lem:e2err-expected-value-bound}
  Let $w_1$ and $w_2$ be two independent standard complex Gaussian random variables\footnote{$z$ is standard complex Gaussian if $z = z_1 + i z_2$ where $z_1$ and $z_2$ are independent standard normal
  random variables.}. Let $
  U = \abs{w_1} w_2 \left(\phase{1+\frac{\sqrt{1-\alpha^2} \conj{w_2}}{\alpha \abs{w_1}}}-1\right).$
Fix $\delta > 0$. Then, there exists a constant $\gamma > 0$ such that if $\sqrt{1-\alpha^2} < \gamma$,
then:
$\ \   \expec{U} \leq (1+\delta) \sqrt{1-\alpha^2}.$
\end{lemma}
See Appendix~\ref{app:sense} for a proof of the above lemma and how we use it to prove Theorem \ref{thm:convergence}.
Combining Theorems~\ref{thm:firststep-svd} and \ref{thm:convergence}, 
we can establish the correctness of Algorithm \ref{algo:phasesensing-nonsparse}.
\begin{theorem}\label{thm:combinedresult}
  Suppose the measurement vectors in \eqref{eqn:magnitude-measurements} are independent standard complex normal vectors.
  There exists a constant $c$ such that if ${m > c n \log n\left(\log^2 n + \log \frac{1}{\epsilon} \log\log\frac{1}{\epsilon}\right)}$ then, with probability greater than $1 - \frac{1}{n}$,
  Algorithm~\ref{algo:phasesensing-nonsparse} outputs $\vecfont{x^{t_0}}$ such that $\dist{\vecfont{x^{t_0}}}{\xo} < \epsilon$, for some global phase choice of $\xo$.
\end{theorem}


\section{Sparse Phase Retrieval}\label{sec:sparse}
In this section, we consider the case where $\xo$ is known to be sparse, with sparsity $k$.
A natural and practical question to ask
here is: can the sample and computational complexity of the recovery algorithm be improved when $k\ll n$. 

Recently, \cite{LiV12} studied this problem for Gaussian $\mat{A}$ and showed that for $\ell_1$ regularized PhaseLift, $m=O(k^2 \log n)$ samples suffice for exact recovery of $\xo$.
However, the computational complexity of this algorithm is still $O(n^3/\epsilon^2)$. 

In this section, we provide a simple extension of our AltMinPhase algorithm that we call SparseAltMinPhase, for the case of sparse $\xo$. The main idea behind our algorithm is to first
recover the support of $\xo$. Then, the problem reduces to phase retrieval of a $k$-dimensional signal.
We then solve the reduced problem using Algorithm~\ref{algo:phasesensing-nonsparse}. The pseudocode for SparseAltMinPhase
is presented in Algorithm~\ref{algo:phasesensing-sparse}. Table \ref{tab:sparse-comparison} provides a comparison of
Algorithm~\ref{algo:phasesensing-sparse} with $\ell_1$-regularized PhaseLift in terms of sample complexity as well as computational complexity. 
\begin{algorithm}[t]
\caption{SparseAltMinPhase}
\label{algo:phasesensing-sparse}
\begin{algorithmic}[1]
\INPUT $\mat{A},\vecfont{y},k$
\STATE $S \leftarrow $ top-$k \; \argmax_{j \in [n]} \sum_{i=1}^m \left|a_{ij} y_i\right|$
\COMMENT{Pick indices of $k$ largest absolute value inner product}
\STATE Apply Algorithm \ref{algo:phasesensing-nonsparse} on $\mat{A}_S,\vecfont{y}_S$ and output the resulting
vector with elements in $S^c$ set to zero.
\end{algorithmic}
\end{algorithm}

\begin{table*}[t]
  \begin{center}
    \begin{tabular}{ | c | c | c |}
      \hline
								& Sample complexity & Comp. complexity  \\ \hline
      \red{Algorithm \ref{algo:phasesensing-sparse}}	& \red{$ \order{k \log n \left( k  + \log^3 k + \log \frac{1}{\epsilon} \log \log \frac{1}{\epsilon}\right)}$}
					&  $\red{\order{k^2 \log n \left(kn + \log^2 \frac{1}{\epsilon} \log \log \frac{1}{\epsilon}\right)}}$ \\ \hline
      $\ell_1$-PhaseLift \cite{LiV12}		& $\order{k^2 \log n}$
					& $\order{n^3/\epsilon^2}$ \\ \hline
      \end{tabular}
      \caption{Comparison of Algorithm~\ref{algo:phasesensing-sparse} with $\ell_1$-PhaseLift when $x^*_{\textrm{min}} = \Omega\left(1/\sqrt{k}\right)$.
      Note that the complexity of Algorithm~\ref{algo:phasesensing-sparse} is dominated by the support finding step. If $k = \order{1}$, Algorithm~\ref{algo:phasesensing-sparse}
      runs in quasi-linear time.}
      \label{tab:sparse-comparison}
  \end{center}\vspace*{-20pt}
\end{table*}
The following lemma shows that if the number of measurements is large enough, step 1 of SparseAltMinPhase
recovers the support of $\xo$ correctly.
\begin{lemma}\label{lem:sparseretrieval-supportrecovery}
  Suppose $\vecfont{x^*}$ is $k$-sparse with support $S$ and $\twonorm{\vecfont{x^*}}=1$. If $\vecfont{a_i}$ are standard
complex Gaussian random vectors and $m > \frac{c}{\left(x_{\textrm{min}}^*\right)^4} \log \frac{n}{\delta}$,
  then Algorithm \ref{algo:phasesensing-sparse} recovers $S$ with probability greater than $1-\delta$, where $x_{\textrm{min}}^*$ is
  the minimum non-zero entry of $\vecfont{x^*}$.
\end{lemma}
The key step of our proof is to show that if $j\in supp(\xo)$, then random variable $Z_{ij}=\sum_{i}|a_{ij}y_i|$ has significantly higher mean than for the case when $j\notin supp(\xo)$. Now, by applying appropriate concentration bounds, we can ensure that $\min_{j\in supp(\xo)}|Z_{ij}| > \max_{j\notin supp(\xo)}|Z_{ij}|$ and hence our algorithm  never picks up an element outside the true support set $supp(\xo)$. See Appendix~\ref{app:sparse} for a detailed proof of the above lemma. 

The correctness of Algorithm \ref{algo:phasesensing-sparse} now is a direct consequence of Lemma \ref{lem:sparseretrieval-supportrecovery}
and Theorem \ref{thm:combinedresult}.
For the special case where each non-zero value in $x^*$ is from $\{-\frac{1}{\sqrt{k}},\frac{1}{\sqrt{k}}\}$, we have the following corollary:
\begin{corollary}
  Suppose $\vecfont{x^*}$ is $k$-sparse with non-zero elements $\pm \frac{1}{\sqrt{k}}$. If the number of measurements
  $m > c \log {n} \left(k^2 + k \log^2 k + k \log \frac{1}{\epsilon}\right)$, then Algorithm \ref{algo:phasesensing-sparse}
  will recover $x^*$ up to accuracy $\epsilon$ with probability greater than $1-\frac{1}{n}$.
\end{corollary}

\section{Experiments}\label{sec:experiments}
In this section, we present experimental evaluation of AltMinPhase (Algorithm~\ref{algo:phasesensing-nonsparse-no-resample})
and compare its performance with the SDP based methods PhaseLift \cite{CandesSV12} and
PhaseCut \cite{WaldspurgerdAM12}. We also empirically demonstrate the advantage of our initialization procedure over
random initialization (denoted by {\bf AltMin (random init)}), which has thus far been considered in the literature \cite{GerchbergS72,Fienup1982,WaldspurgerdAM12,CandesESV13}.  {\bf AltMin (random init)} is the same as AltMinPhase except that step 1 of Algorithm~\ref{algo:phasesensing-nonsparse-no-resample}
is replaced with:$\vecfont{x^0} \leftarrow$ Uniformly random vector from the unit sphere.

In the noiseless setting, a trial is said to {succeed} if the output $\vecfont{x}$ satisfies $\twonorm{\vecfont{x}-\xo} < 10^{-2}$.
For a given dimension, we do a linear search for smallest $m$ (number of samples) such that empirical success ratio over $20$ runs is at least $0.8$.
We implemented our methods in Matlab, while we obtained the code for PhaseLift and PhaseCut from the authors of \cite{OhlssonYDS11} and \cite{WaldspurgerdAM12} respectively.

We now present results from our experiments in three different settings. 

{\bf Independent Random Gaussian Measurements}:
Each measurement vector $\vecfont{a_i}$ is generated from the standard complex Gaussian distribution.
This measurement scheme was first suggested by \cite{CandesSV12} as a first step to obtain a theoretical understanding of the problem.

{\bf Multiple Random Illumination Filters}: We now present our results for the setting where the measurements are obtained using multiple illumination filters; this setting was suggested by \cite{CandesESV13}. 
In particular, {choose $J$ vectors $\vecfont{z^{(1)}},\cdots,\vecfont{z^{(J)}}$ and compute the following discrete Fourier transforms:
\begin{align*}
  \vecfont{\widehat{x}^{(u)}} = \textrm{DFT}\left(\vecfont{x^*} \cdot * \; \vecfont{z^{(u)}}\right),
\end{align*}
where $\cdot *$ denotes component-wise multiplication. Our measurements will then be the magnitudes of components of the vectors $\vecfont{\widehat{x}^{(1)}},\cdots,\vecfont{\widehat{x}^{(J)}}$. Note that this gives a total of $Jn$ measurements.} The above measurement scheme can be implemented by modulating the light beam or by the use of masks; see \cite{CandesESV13} for more details.

For this setting, we conduct a similar set of experiments as the previous setting.  That is, we vary dimensionality of the true signal $\vecfont{z^{(u)}}$ (generated from the Gaussian distribution)and then empirically determine measurement  and computational cost of each algorithm. Figures~\ref{fig:randomgaussianfilters} (a) and (b) present our experimental results for this measurement scheme. Here again, we make similar observations as the last setting. That is, the measurement complexity of AltMinPhase is similar to PhaseCut and PhaseLift, but AltMinPhase is orders of magnitude faster than PhaseLift and PhaseCut. Note that Figure~\ref{fig:randomgaussianfilters} is on a log-scale.

\begin{figure*}[!ht]
  \centering
  \begin{tabular}[t]{cc}
    \includegraphics[width=.48\textwidth]{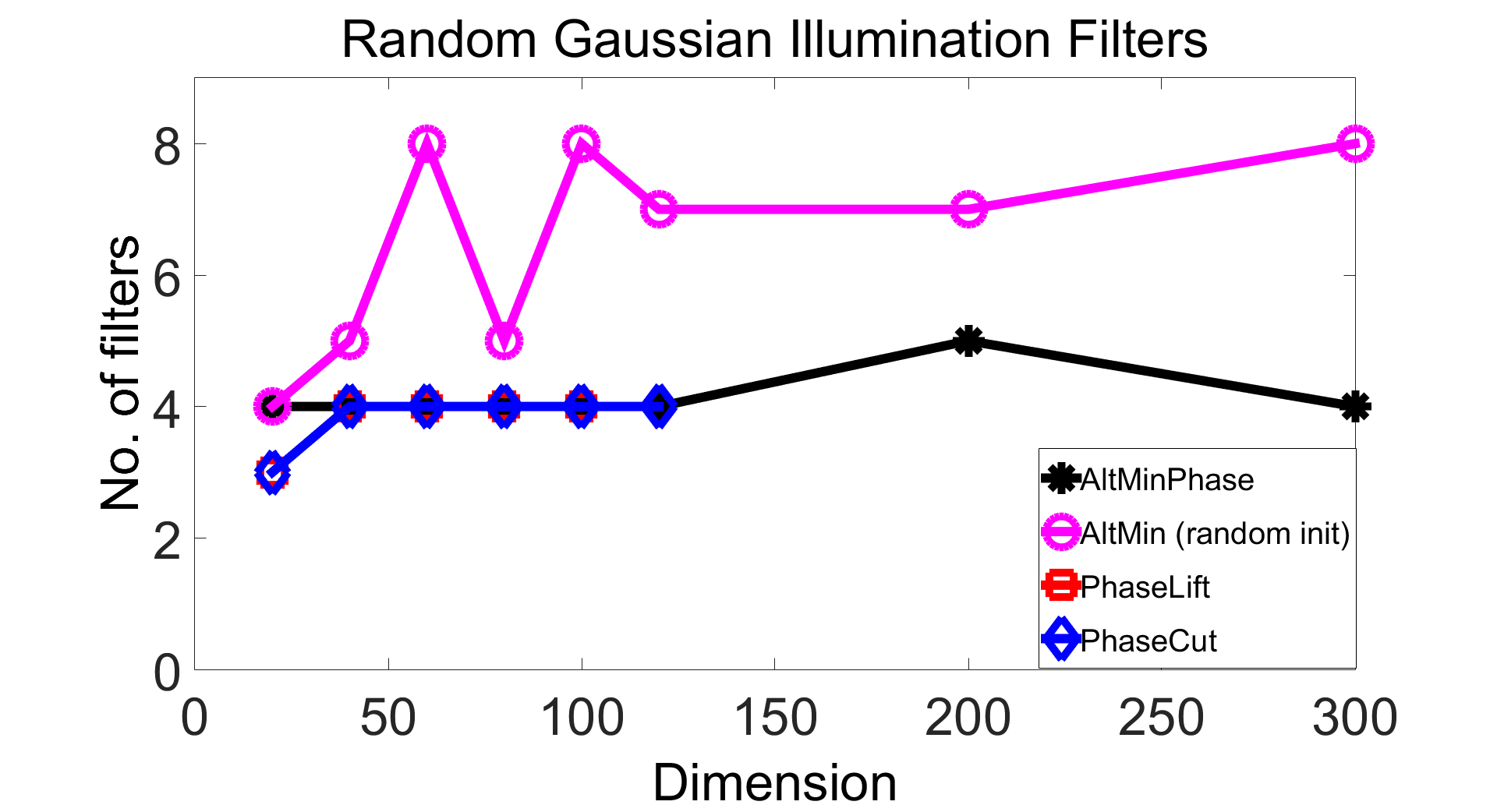}&\includegraphics[width=.48\textwidth]{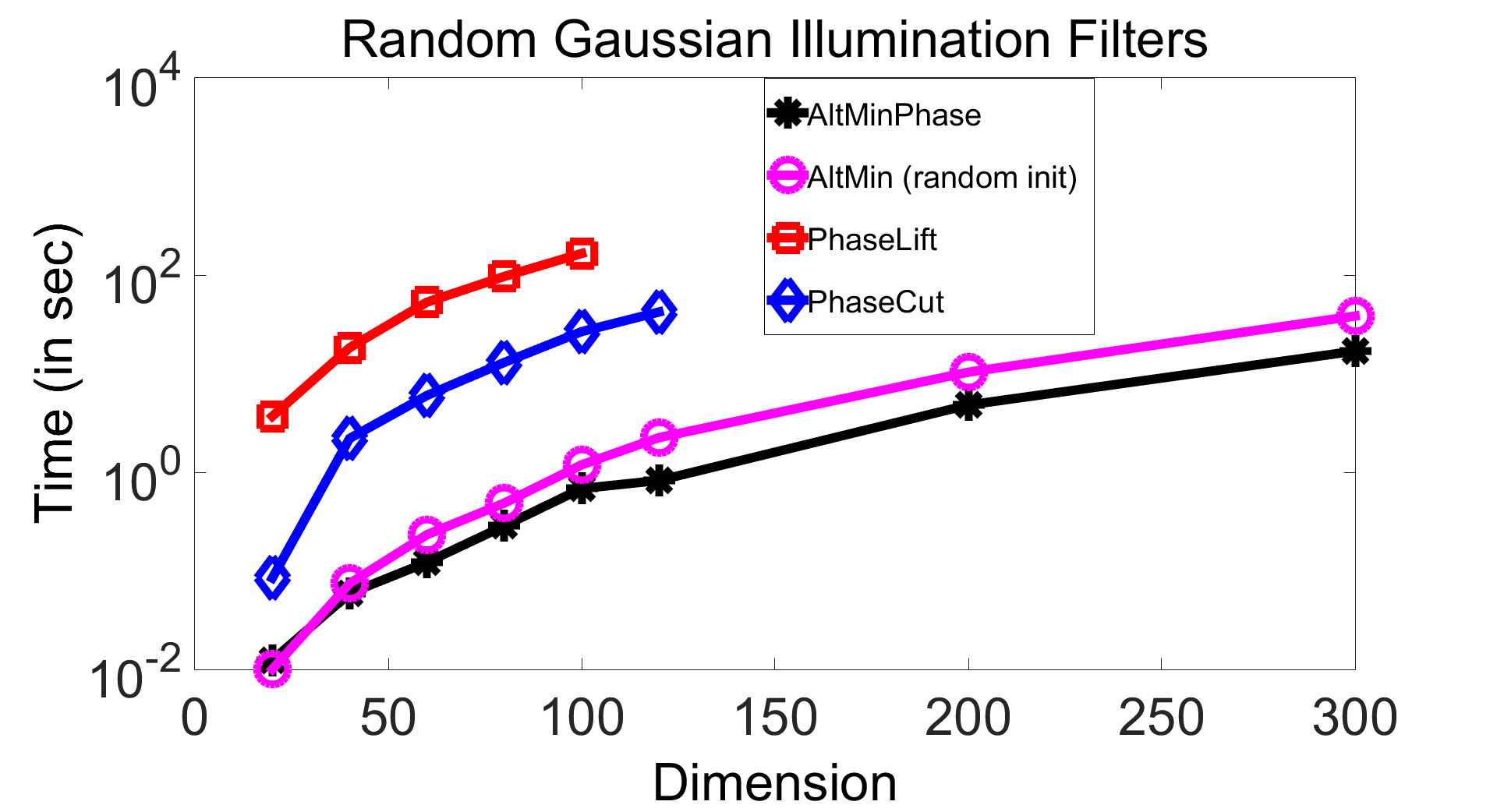}\\
{\bf (a)}&{\bf (b)}
  \end{tabular}
  \caption{Sample and time complexity for successful recovery using random Gaussian illumination filters. Similar to Figure~\ref{fig:sense}, we observe that
  AltMinPhase has similar number of filters ($J$) as PhaseLift and PhaseCut, but is computationally much more efficient. We also see that AltMinPhase performs better than
  AltMin (randominit).}
  \label{fig:randomgaussianfilters}
\end{figure*}

{\bf Noisy Phase Retrieval}: Finally, we study our method in the following noisy measurement scheme:
\begin{align}\label{eqn:magnitudemeasurements-noisy}
y_i ~ = ~ | \langle \vecfont{a}_i , \vecfont{x^*} + w_i \rangle | \quad \quad \quad \text{for $i = 1,\ldots,m$},
\end{align}
where $w_i$ is the noise in the $i$-th measurement and is sampled from $\mathcal{N}(0,\sigma^2)$. We fix $n=64$ and $m=6n$. We then vary the amount of noise added $\sigma$ and measure the $\ell_2$ error
in recovery, i.e., $\|\x-\xo\|_2$, where $\x$ is the recovered vector. Figure~\ref{fig:noisy-initialization}(a) compares the performance of various methods with varying amount of noise.
We observe that our method outperforms PhaseLift and has similar recovery error as PhaseCut.

{\bf Geometric Decay}: Finally, we provide empirical results verifying that AltMinPhase reduces the error at a geometric rate as guaranteed by Theorem~\ref{thm:convergence} but no faster.
The measurement vectors were chosen to be standard complex Gaussian with $n = 64$ and $m=6n$.
Figure~\ref{fig:noisy-initialization}(b) shows the plot of empirical error vs the number of iterations.

\begin{figure*}[!ht]
  \centering
  \begin{tabular}[t]{cc}
    \includegraphics[width=.48\textwidth]{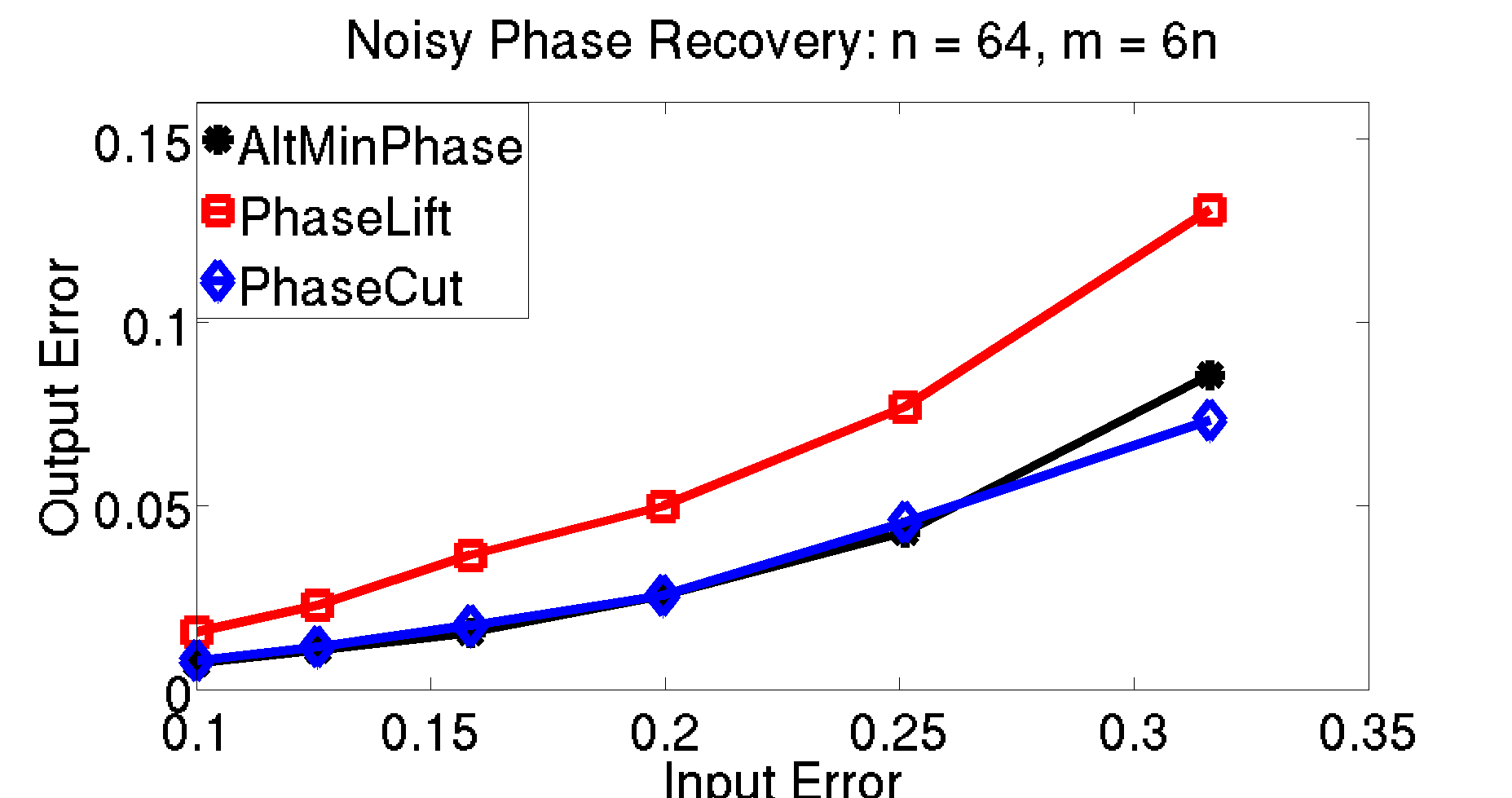}& \includegraphics[width=0.48\textwidth]{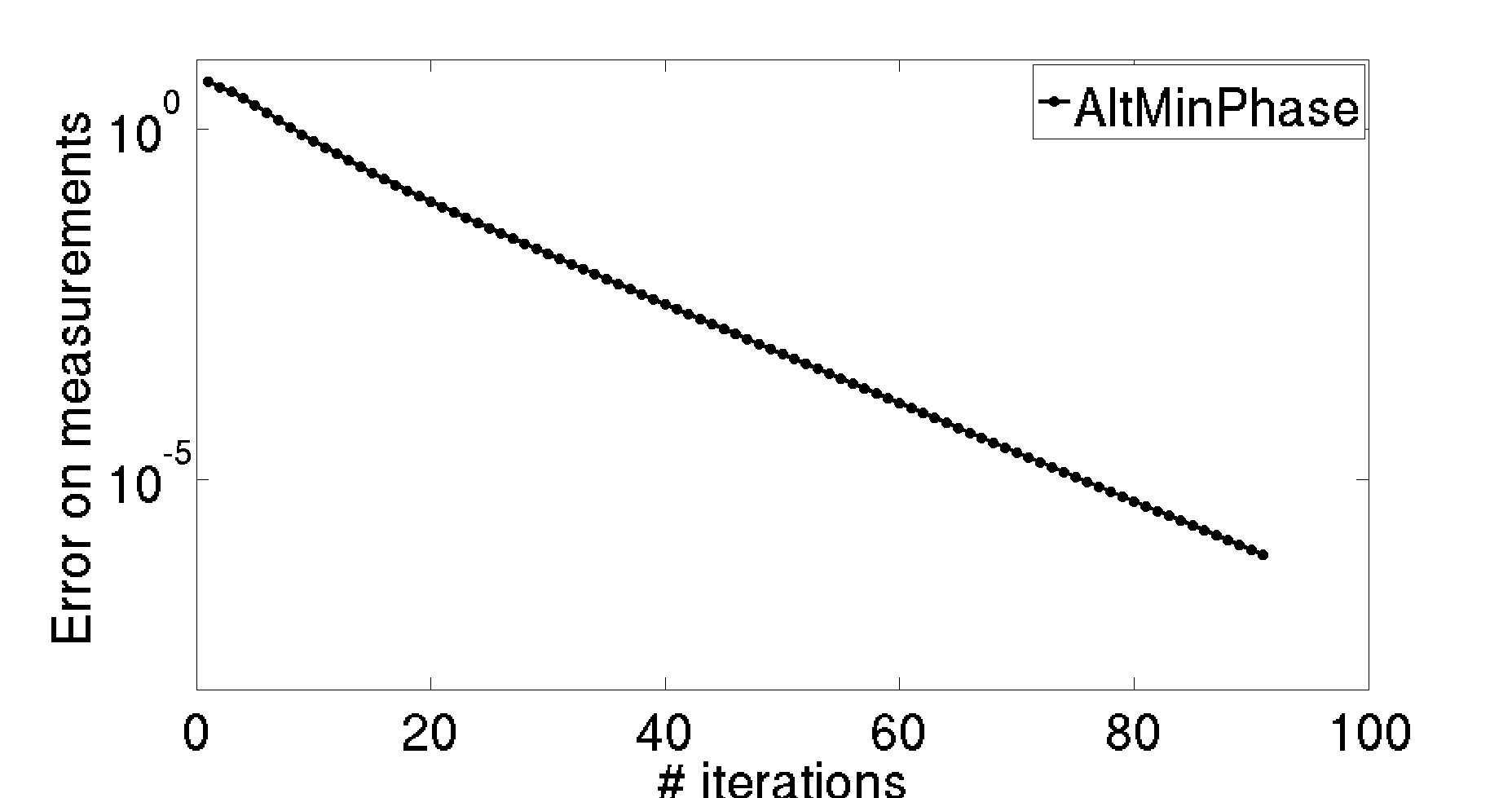}\\
{\bf (a)}&{\bf (b)}
  \end{tabular}
  \caption{{\bf (a)}: Recovery error $\|\x-\xo\|_2$ incurred by various methods with increasing amount of noise ($\sigma$). AltMinPhase and PhaseCut perform comparably while PhaseLift incurs significantly larger error.
  {\bf (b)}: Plot of empirical error $\twonorm{y - \abs{\mat{A}^T\vecfont{x}}}$ vs number of iterations for AltMinPhase. Each entry of $\mat{A}$ is chosen to be standard complex Gaussian with $n = 64$ and
  $m = 6n$. We can see that the error decreases geometrically suggesting that Theorem~\ref{thm:convergence} is tight in some sense.}
  \label{fig:noisy-initialization}
\end{figure*}


\section*{Acknowledgment}
S. Sanghavi would like to acknowledge support from NSF grants 1302435 and 0954059.
%
\clearpage
\newpage

\begin{thebibliography}{10}

\bibitem{AbrahamsL96}
J.~Abrahams and A.~Leslie.
\newblock Methods used in the structure determination of bovine mitochondrial
  f1 atpase.
\newblock {\em Acta Crystallographica Section D: Biological Crystallography},
  52(1):30--42, 1996.

\bibitem{AgarwalANJ2013}
A.~Agarwal, A.~Anandkumar, P.~Jain, and P.~Netrapalli.
\newblock Learning sparsely used overcomplete dictionaries via alternating
  minimization.
\newblock {\em arXiv preprint arXiv:1310.7991}, 2014.

\bibitem{BauschkeCL03}
H.~H. Bauschke, P.~L. Combettes, and D.~R. Luke.
\newblock Hybrid projection--reflection method for phase retrieval.
\newblock {\em JOSA A}, 20(6):1025--1034, 2003.

\bibitem{BianSZGCD2014}
L.~Bian, J.~Suo, G.~Zheng, K.~Guo, F.~Chen, and Q.~Dai.
\newblock Fourier ptychographic reconstruction using wirtinger flow
  optimization.
\newblock {\em arXiv preprint arXiv:1411.6431}, 2014.

\bibitem{Bregman65}
L.~Bregman.
\newblock Finding the common point of convex sets by the method of successive
  projection.(russian).
\newblock In {\em Dokl. Akad. Nauk SSSR}, volume 162, pages 487--490, 1965.

\bibitem{BruckS79}
Y.~M. Bruck and L.~Sodin.
\newblock On the ambiguity of the image reconstruction problem.
\newblock {\em Optics Communications}, 30(3):304--308, 1979.

\bibitem{CandesLS14}
E.~Cand{\`e}s, X.~Li, and M.~Soltanolkotabi.
\newblock Phase retrieval via wirtinger flow: Theory and algorithms.
\newblock {\em arXiv preprint arXiv:1407.1065}, 2014.

\bibitem{CandesESV13}
E.~J. Cand{\`e}s, Y.~C. Eldar, T.~Strohmer, and V.~Voroninski.
\newblock Phase retrieval via matrix completion.
\newblock {\em SIAM Journal on Imaging Sciences}, 6(1):199--225, 2013.

\bibitem{CandesL12}
E.~J. Cand{\`e}s and X.~Li.
\newblock Solving quadratic equations via phaselift when there are about as
  many equations as unknowns.
\newblock {\em Foundations of Computational Mathematics}, 14(5):1017--1026,
  2014.

\bibitem{CandesLS13}
E.~J. Candes, X.~Li, and M.~Soltanolkotabi.
\newblock Phase retrieval from coded diffraction patterns.
\newblock {\em Applied and Computational Harmonic Analysis}, 2014.

\bibitem{CandesSV12}
E.~J. Cand{\`e}s, T.~Strohmer, and V.~Voroninski.
\newblock Phaselift: Exact and stable signal recovery from magnitude
  measurements via convex programming.
\newblock {\em Communications on Pure and Applied Mathematics}, 2012.

\bibitem{ChaiMP11}
A.~Chai, M.~Moscoso, and G.~Papanicolaou.
\newblock Array imaging using intensity-only measurements.
\newblock {\em Inverse Problems}, 27(1):015005, 2011.

\bibitem{ChiRS05}
T.~Chi, P.~Ru, and S.~A. Shamma.
\newblock Multiresolution spectrotemporal analysis of complex sounds.
\newblock {\em The Journal of the Acoustical Society of America}, 118:887,
  2005.

\bibitem{DaintyF87}
J.~C. Dainty and J.~R. Fienup.
\newblock Phase retrieval and image reconstruction for astronomy.
\newblock {\em Image Recovery: Theory and Application, ed. byH. Stark, Academic
  Press, San Diego}, pages 231--275, 1987.

\bibitem{Duadietal11}
H.~Duadi, O.~Margalit, V.~Mico, J.~A. Rodrigo, T.~Alieva, J.~Garcia, and
  Z.~Zalevsky.
\newblock Digital holography and phase retrieval.
\newblock {\em Source: Holography, Research and Technologies. InTech}, 2011.

\bibitem{EldarM12}
Y.~C. Eldar and S.~Mendelson.
\newblock Phase retrieval: Stability and recovery guarantees.
\newblock {\em Applied and Computational Harmonic Analysis}, 2013.

\bibitem{Elser03}
V.~Elser.
\newblock Phase retrieval by iterated projections.
\newblock {\em JOSA A}, 20(1):40--55, 2003.

\bibitem{FienupMSS93}
J.~Fienup, J.~Marron, T.~Schulz, and J.~Seldin.
\newblock Hubble space telescope characterized by using phase-retrieval
  algorithms.
\newblock {\em Applied optics}, 32(10):1747--1767, 1993.

\bibitem{Fienup1982}
J.~R. Fienup et~al.
\newblock Phase retrieval algorithms: a comparison.
\newblock {\em Applied optics}, 21(15):2758--2769, 1982.

\bibitem{Gabor48}
D.~Gabor.
\newblock A new microscopic principle.
\newblock {\em Nature}, 161(4098):777--778, 1948.

\bibitem{GerchbergS72}
R.~W. Gerchberg and W.~O. Saxton.
\newblock A practical algorithm for the determination of phase from image and
  diffraction plane pictures.
\newblock {\em Optik}, 35:237, 1972.

\bibitem{GrossKK13}
D.~Gross, F.~Krahmer, and R.~Kueng.
\newblock A partial derandomization of phaselift using spherical designs.
\newblock {\em Journal of Fourier Analysis and Applications}, pages 1--38,
  2014.

\bibitem{Hardt13}
M.~Hardt.
\newblock Understanding alternating minimization for matrix completion.
\newblock In {\em Foundations of Computer Science (FOCS), 2014 IEEE 55th Annual
  Symposium on}, pages 651--660. IEEE, 2014.

\bibitem{Hayes82}
M.~Hayes.
\newblock The reconstruction of a multidimensional sequence from the phase or
  magnitude of its fourier transform.
\newblock {\em Acoustics, Speech and Signal Processing, IEEE Transactions on},
  30(2):140--154, 1982.

\bibitem{HsiehCD12}
C.-J. Hsieh, K.-Y. Chiang, and I.~S. Dhillon.
\newblock Low rank modeling of signed networks.
\newblock In {\em KDD}, pages 507--515, 2012.

\bibitem{Hurt01}
N.~E. Hurt.
\newblock {\em Phase Retrieval and Zero Crossings: Mathematical Methods in
  Image Reconstruction}, volume~52.
\newblock Kluwer Academic Print on Demand, 2001.

\bibitem{JaganathanOH12}
K.~Jaganathan, S.~Oymak, and B.~Hassibi.
\newblock Recovery of sparse 1-d signals from the magnitudes of their fourier
  transform.
\newblock In {\em Information Theory Proceedings (ISIT), 2012 IEEE
  International Symposium On}, pages 1473--1477. IEEE, 2012.

\bibitem{JainN2014}
P.~Jain and P.~Netrapalli.
\newblock Fast exact matrix completion with finite samples.
\newblock In {\em Conference on Learning Theory (COLT)}, 2015.

\bibitem{JainNS12}
P.~Jain, P.~Netrapalli, and S.~Sanghavi.
\newblock Low-rank matrix completion using alternating minimization.
\newblock In {\em Proceedings of the forty-fifth annual ACM symposium on Theory
  of computing}, pages 665--674. ACM, 2013.

\bibitem{KeshavanOM2009Noisy}
R.~Keshavan, A.~Montanari, and S.~Oh.
\newblock Matrix completion from noisy entries.
\newblock In {\em Advances in Neural Information Processing Systems}, pages
  952--960, 2009.

\bibitem{Keshavan12}
R.~H. Keshavan.
\newblock Efficient algorithms for collaborative filtering.
\newblock Phd Thesis, Stanford University, 2012.

\bibitem{KeshavanOM2009}
R.~H. Keshavan, A.~Montanari, and S.~Oh.
\newblock Matrix completion from a few entries.
\newblock {\em IEEE Transactions on Information Theory}, 56(6):2980--2998,
  2010.

\bibitem{KimPar08}
H.~Kim and H.~Park.
\newblock Nonnegative matrix factorization based on alternating nonnegativity
  constrained least squares and active set method.
\newblock {\em SIAM J. Matrix Anal. Appl.}, 30(2):713--730, July 2008.

\bibitem{KimPark08b}
J.~Kim and H.~Park.
\newblock Sparse nonnegative matrix factorization for clustering.
\newblock Technical Report GT-CSE-08-01, Georgia Institute of Technology, 2008.

\bibitem{LeithU62}
E.~N. Leith and J.~Upatnieks.
\newblock Reconstructed wavefronts and communication theory.
\newblock {\em JOSA}, 52(10):1123--1128, 1962.

\bibitem{LiW09}
W.~V. Li and A.~Wei.
\newblock Gaussian integrals involving absolute value functions.
\newblock In {\em Proceedings of the Conference in Luminy}, 2009.

\bibitem{LiV12}
X.~Li and V.~Voroninski.
\newblock Sparse signal recovery from quadratic measurements via convex
  programming.
\newblock {\em SIAM Journal on Mathematical Analysis}, 45(5):3019--3033, 2013.

\bibitem{LorentzvGM1996}
G.~G. Lorentz, M.~von Golitschek, and Y.~Makovoz.
\newblock {\em Constructive approximation: advanced problems}, volume 304.
\newblock Springer Berlin, 1996.

\bibitem{Marchesini07}
S.~Marchesini.
\newblock Invited article: A unified evaluation of iterative projection
  algorithms for phase retrieval.
\newblock {\em Review of Scientific Instruments}, 78(1):011301--011301, 2007.

\bibitem{Marchesini07b}
S.~Marchesini.
\newblock Phase retrieval and saddle-point optimization.
\newblock {\em JOSA A}, 24(10):3289--3296, 2007.

\bibitem{MiaoCKS99}
J.~Miao, P.~Charalambous, J.~Kirz, and D.~Sayre.
\newblock Extending the methodology of x-ray crystallography to allow imaging
  of micrometre-sized non-crystalline specimens.
\newblock {\em Nature}, 400(6742):342--344, 1999.

\bibitem{Miaoetal02}
J.~Miao, T.~Ishikawa, B.~Johnson, E.~H. Anderson, B.~Lai, and K.~O. Hodgson.
\newblock High resolution 3d x-ray diffraction microscopy.
\newblock {\em Physical review letters}, 89(8):088303, 2002.

\bibitem{Millane1990}
R.~Millane.
\newblock Phase retrieval in crystallography and optics.
\newblock {\em JOSA A}, 7(3):394--411, 1990.

\bibitem{Misell73}
D.~Misell.
\newblock A method for the solution of the phase problem in electron
  microscopy.
\newblock {\em Journal of Physics D: Applied Physics}, 6(1):L6, 1973.

\bibitem{NetrapalliNSAJ2014}
P.~Netrapalli, U.~Niranjan, S.~Sanghavi, A.~Anandkumar, and P.~Jain.
\newblock Non-convex robust pca.
\newblock In {\em Advances in Neural Information Processing Systems}, pages
  1107--1115, 2014.

\bibitem{OhlssonYDS11}
H.~Ohlsson, A.~Y. Yang, R.~Dong, and S.~S. Sastry.
\newblock Compressive phase retrieval from squared output measurements via
  semidefinite programming.
\newblock {\em arXiv preprint arXiv:1111.6323}, 2011.

\bibitem{OymakJFEH12}
S.~Oymak, A.~Jalali, M.~Fazel, Y.~C. Eldar, and B.~Hassibi.
\newblock Simultaneously structured models with application to sparse and
  low-rank matrices.
\newblock {\em arXiv preprint arXiv:1212.3753}, 2012.

\bibitem{Sanz85}
J.~L. Sanz.
\newblock Mathematical considerations for the problem of fourier transform
  phase retrieval from magnitude.
\newblock {\em SIAM Journal on Applied Mathematics}, 45(4):651--664, 1985.

\bibitem{ShechtmanBE13}
Y.~Shechtman, A.~Beck, and Y.~C. Eldar.
\newblock Gespar: Efficient phase retrieval of sparse signals.
\newblock {\em Signal Processing, IEEE Transactions on}, 62(4):928--938, 2014.

\bibitem{ShechtmanESS11}
Y.~Shechtman, Y.~C. Eldar, A.~Szameit, and M.~Segev.
\newblock Sparsity based sub-wavelength imaging with partially incoherent light
  via quadratic compressed sensing.
\newblock {\em Optics express}, 19(16):14807--14822, 2011.

\bibitem{Tropp11}
J.~A. Tropp.
\newblock User-friendly tail bounds for sums of random matrices.
\newblock {\em Foundations of Computational Mathematics}, 12(4):389--434, 2012.

\bibitem{TrussellC84}
H.~Trussell and M.~Civanlar.
\newblock The feasible solution in signal restoration.
\newblock {\em Acoustics, Speech and Signal Processing, IEEE Transactions on},
  32(2):201--212, 1984.

\bibitem{Vershynin10}
R.~Vershynin.
\newblock Introduction to the non-asymptotic analysis of random matrices.
\newblock {\em arXiv preprint arXiv:1011.3027}, 2010.

\bibitem{WaldspurgerdAM12}
I.~Waldspurger, A.~d’Aspremont, and S.~Mallat.
\newblock Phase recovery, maxcut and complex semidefinite programming.
\newblock {\em Mathematical Programming}, 149(1-2):47--81, 2015.

\bibitem{YoulaW82}
D.~C. Youla and H.~Webb.
\newblock Image restoration by the method of convex projections: Part
  1ߞtheory.
\newblock {\em Medical Imaging, IEEE Transactions on}, 1(2):81--94, 1982.

\bibitem{ZouHT06}
H.~Zou, T.~Hastie, and R.~Tibshirani.
\newblock Sparse principal component analysis.
\newblock {\em JCGS}, 15(2):262--286, 2006.

\end{thebibliography}

\clearpage
\newpage
\appendix
\section{Proofs for Section \ref{sec:sense_analysis}}\label{app:sense}
\subsection{Proof of the Initialization Step}
\begin{proof}[Proof of Theorem \ref{thm:firststep-svd}]
Recall that $\vecfont{x^0}$ is the top singular vector of $\mat{S}=\frac{1}{n}\sum_\ell |\vecfont{a_{\ell}}^T\xo|^2 \vecfont{a_{\ell}} \vecfont{a_{\ell}}^T.$ As $\vecfont{a_{\ell}}$ are
rotationally invariant random variables, wlog, we can assume that $\xo=\vecfont{e_1}$ where $\vecfont{e_1}$ is the first canonical basis vector. Also note that
$\expec{|\iprod{\vecfont{a}}{\vecfont{e_1}}|^2 \vecfont{a}\vecfont{a}^T}=\mat{D}$, where $\mat{D}$ is a diagonal matrix with $D_{11}=\E_{a \sim \normal_C(0,1)}[|a|^4]= {8}$ and $D_{ii}=\E_{a\sim \normal_C(0,1), b\sim \normal_C(0,1)}[|a|^2|b|^2]={4}, \forall i>1$.

We break our proof of the theorem into two steps: \\[2pt]
{\bf (1)}: Show that, with probability $> 1-\frac{4}{m^2}$: $\|\mat{S}-\mat{D}\|_2 < c/4$. \\[2pt]
{\bf (2)}: Use (1) to prove the theorem. \\
{\bf Proof of Step (2)}: We have $\abs{\ip{\vecfont{x^0}}{\mat{S}\vecfont{x^0}}} \leq c/4 + 8\abs{\iprod{\vecfont{x^0}}{\e_1}}^2+4 \sum_{i=2}^{n} \abs{\vecfont{x^0}_i}^2 = c/4 + 4\abs{\vecfont{x^0}_1}^2+4$. On the other hand, since $\vecfont{x^0}$ is the top singular value of $\mat{S}$, by using triangle inequality, we have
$\abs{\ip{\vecfont{x^0}}{\mat{S}\vecfont{x^0}}} > 8-c/4$. Hence,  $\abs{\ip{\vecfont{x^0}}{\e_1}}^2> 1-\frac{c}{8}$.
This yields {$\distop{\vecfont{x^0}}{\xo} = \sqrt{1-\ip{\vecfont{x^0}}{\e_1}^2} < \sqrt{c}$}.

{\bf Proof of Step (1)}: We now complete our proof by proving (1). To this end, we use the following matrix concentration result from \cite{Tropp11}:
\begin{theorem}[Theorem 1.5 of \cite{Tropp11}]\label{thm:tropp}
Consider a finite sequence $\mat{X_i}$ of  self-adjoint independent random matrices with dimensions $n\times n$. Assume that $\mathbb{E}[\mat{X_i}] = 0$ and $\twonorm{\mat{X_i}} \leq R, \forall i$,  almost surely.
Let {$\sigma^2 := \|\sum_i\mathbb{E}[\mat{X_i}^2]\|_2$}. Then the following holds $\forall \nu\geq 0$:  $$P\left(\|\frac{1}{m}\sum_{i=1}^m\mat{X_i}\|_2\geq \nu\right)\leq 2n\exp\left(\frac{-m^2\nu^2}{\sigma^2+Rm\nu/3}\right).$$
\end{theorem}

Note that Theorem~\ref{thm:tropp} assumes $\max_\ell |a_{1\ell}|^2\|\a_\ell\|^2$ to be bounded, where $a_{1\ell}$ is the first component of $\a_\ell$. However, $\a_{\ell}$ is a normal random variable and hence can be unbounded. We address this issue by observing that probability that $\Pr(\|\a_\ell\|^2\geq 2n\ OR\ |a_{1\ell}|^2\geq 2\log m)\leq 2\exp(-n/2)+\frac{1}{m^2}.$ Hence, for large enough $n, \widehat{c}$ and $m>\widehat{c}n$, w.p. $1-\frac{3}{m^2}$, \begin{equation}\label{eq:abound}\max_\ell |a_{1\ell}|^2\|\a_\ell\|^2\leq 4n \log(m).\end{equation}
Now, consider  truncated random variable $\widetilde{\a}_\ell$ s.t. $\widetilde{\a}_\ell=\a_\ell\ \  if\ \  |a_{1\ell}|^2\leq 2\log(m) \& \|\a_\ell\|^2\leq 2n$ and $\widetilde{\a}_\ell=0$ otherwise. 
Now, note that $\widetilde{\a}_\ell$ is symmetric around origin and also $\E[\widetilde{a}_{i\ell}\widetilde{a}_{j\ell}]=0, \forall i\neq j$. Also, $\E[|\widetilde{a}_{i\ell}|^2]\leq 1$. Hence, $\|\E[|\widetilde{a}_{1\ell}|^2\|\widetilde{\a}_\ell\|^2\widetilde{\a}_\ell\widetilde{\a}_\ell^\dag]\|_2\leq 4n\log(m)$. Now, applying Theorem~\ref{thm:tropp} given above, we get (w.p. $\geq 1-1/m^2$)
\begin{equation*}
  \|\frac{1}{m}\sum_\ell|\widetilde{a}_{1\ell}|^2\widetilde{\a}_\ell\widetilde{\a}_\ell^\dag -\E[|\widetilde{a}_{1\ell}|^2\widetilde{\a}_\ell\widetilde{\a}_\ell^\dag]\|_2\leq \frac{4n\log^{3/2}(m)}{\sqrt{m}}. 
\end{equation*}
Furthermore, $\a_\ell=\widetilde{\a}_\ell$ with probability larger than $1-\frac{3}{m^2}$. Hence, w.p. $\geq 1-\frac{4}{m^2}$: 
\begin{equation*}
  \|S -\E[|\widetilde{\a}_\ell^1|^2\widetilde{\a}_\ell\widetilde{\a}_\ell^\dag]\|_2\leq \frac{4n\log^{3/2}(m)}{\sqrt{m}}. 
\end{equation*}
Now, the remaining task is to show that $\|\E[|\widetilde{\a}_\ell^1|^2\widetilde{\a}_\ell\widetilde{\a}_\ell^\dag]-\E[|{\a}_\ell^1|^2{\a}_\ell{\a}_\ell^\dag]\|_2\leq \frac{1}{m}$. This follows easily by observing that $\E[\widetilde{\a}_\ell^i\widetilde{\a}_\ell^j]=0$ and by bounding $\E[|\widetilde{\a}_\ell^1|^2|\widetilde{\a}_\ell^i|^2-|{\a}_\ell^1|^2|{\a}_\ell^i|^2\leq 1/m$ by using a simple second and fourth moment calculations for the normal distribution. 

\end{proof}

\subsection{Proof of per step reduction in error}
In all the lemmas in this section, $\delta$ is a small numerical constant (can be taken to be $0.01$).
\begin{lemma}\label{lem:xo_comp}
  Assume the hypothesis of Theorem \ref{thm:convergence} and let $\xplus$ be as defined in \eqref{eq:xupdate}. Then, there exists an absolute numerical constant $c$ such that the following holds
  (w.p. $\geq 1-\frac{\eta}{4}$):  $\twonorm{\left(\mat{A}\mat{A}^T\right)^{-1}\mat{A}\left(\mat{D}-\mat{I}\right)\mat{A}^T\xo} < c \distop{\xo}{\vecfont{x}}. $
  Furthermore, we have:
  \begin{align*}
  \twonorm{\frac{1}{2m}\mat{A}\mat{A}^T - \mat{I}} &< \frac{1}{\sqrt{\widehat{c}}}, \\
  \twonorm{\frac{1}{\sqrt{2m}}\mat{A}} &< 1+2/\sqrt{\widehat{c}}, \mbox{ and }\\
  \twonorm{\left(\mat{D}-\mat{I}\right)\mat{A}^T\vecfont{\xo}} &< c \sqrt{m}\distop{\xo}{\vecfont{x^t}}.
  \end{align*}
\end{lemma}
\begin{proof}
Using \eqref{eq:t2} and the fact that $\|\xo\|_2=1$, $\x^{*T}\xplus =1 + \x^{*T}\left(\mat{A}\mat{A}^T\right)^{-1}\mat{A}\left(\mat{D}-\mat{I}\right)\mat{A}^T\xo$. That is, $|\x^{*T}\xplus|\geq 1-\|\left(\frac{1}{2m}\mat{A}\mat{A}^T\right)^{-1}\|_2\|\frac{1}{\sqrt{2m}}A\|_2\|\frac{1}{\sqrt{2m}}\left(\mat{D}-\mat{I}\right)\mat{A}^T\xo\|_2$.
Assuming $m>\widehat{c} \log \frac{1}{\eta} n$,
Standard results in random matrix theory\cite{Vershynin10} tell us that
$\twonorm{\frac{1}{2m}\mat{A}\mat{A}^T - \mat{I}} < \frac{1}{\sqrt{\widehat{c}}}$,
wp $\geq 1-\frac{\eta}{10}$.
This means that
$\|\left(\frac{1}{2m}\mat{A}\mat{A}^T\right)^{-1}\|_2\leq 1/(1-2/\sqrt{\widehat{c}})^2$ and $\|\mat{A}\|_2\leq 1+2/\sqrt{\widehat{c}}$.
Note that both the quantities can be bounded by constants that are close to $1$ by selecting a large enough $\widehat{c}$.
Also note that $\frac{1}{2m}\mat{A}\mat{A}^T$ converges to $\mat{I}$ (the identity matrix), or equivalently $\frac{1}{m}\mat{A}\mat{A}^T$ converges to $2\mat{I}$ since the elements of $A$ are standard
normal complex random variables and not standard normal real random variables.

The key challenge now is to bound $\twonorm{\left(\mat{D}-\mat{I}\right)\mat{A}^T\vecfont{\xo}}$ by $ c \sqrt{m}\distop{\xo}{\vecfont{x^t}}$ for a global constant $c>0$. 
Note that since \eqref{eq:t2} is invariant with respect to $\twonorm{\vecfont{x^t}}$, we can assume that $\twonorm{\vecfont{x^t}}=1$.
Note further that, since the distribution of $\mat{A}$ is rotationally invariant and is independent of $\xo$ and $\vecfont{x^t}$, wlog, we can assume that $\xo=\e_1$ and
$\vecfont{x^t}=\alpha \vecfont{e_1}+\sqrt{1-\alpha^2}\vecfont{e_2}$, where $\alpha=\ip{\vecfont{x^t}}{\xo} \geq 0$.
A subtle thing to keep in mind here is that $\alpha$, being the inner product of $\vecfont{x^t}$ and $\xo$, is in general complex. However, we recall from the assumption in our theorem that we choose the global phase factor of $\xo$ such that $\alpha = \ip{\vecfont{x^t}}{\xo} \geq 0$.
Making the notation
\begin{align}
\hspace{-0.2cm}
  U_l \eqdef \abs{a_{1l}}^2 \abs{\phase{\left(\alpha \conj{a}_{1l}+\sqrt{1-\alpha^2} \conj{a}_{2l}\right) a_{1l}}-1}^2
  \label{eq:ul}
\end{align}
gives us $\twonorm{\left(\mat{D}-\mat{I}\right)\mat{A}^T\vecfont{e_1}}^2 = \sum_{l=1}^m U_\ell$.

Using Lemma \ref{lem:ul} finishes the proof.
\end{proof}
The following lemma, Lemma \ref{lem:ul} shows that if $U_\ell$ are as defined in Lemma \ref{lem:xo_comp} then, the sum of $U_\ell, 1\leq \ell\leq m$ concentrates well around
$\expec{U_\ell}$ and also $\expec{U_\ell}\leq c\sqrt{m}\distop{\xo}{\vecfont{x^t}}$. The proof of Lemma~\ref{lem:ul} requires careful analysis as it provides tail bound and
expectation bound of a random variable that is a product of correlated sub-exponential complex random variables.
\begin{lemma}\label{lem:ul}
Assume the hypothesis of Lemma~\ref{lem:xo_comp}.
Let $U_\ell$ be as defined in \eqref{eq:ul} and let each $a_{1l}, a_{2l}, \forall 1\leq l\leq m$ be sampled from  standard normal distribution for complex numbers. Then, with probability greater than $1 - \frac{\eta}{4}$, we have: 
$\sum_{l=1}^m U_l	 \leq c^2m(1-\alpha^2),$ for a global constant $c>0$. 
\end{lemma}
\begin{proof}
We first estimate $\prob{U_l > t}$ so as to:
\begin{enumerate}
  \item	Calculate $\expec{U_l}$ and,
  \item	Show that $U_l$ is a subexponential random variable and use that fact to derive concentration bounds.
\end{enumerate}
{In what follows, we use $c$ to denote a numerical constant whose value may change from line to line.}
$ \prob{U_l > t} = \int_{\frac{\sqrt{t}}{2}}^{\infty} p_{\abs{a_{1l}}}(s) \prob{W_l > \frac{\sqrt{t}}{s} \middle| \abs{a_{1l}} } ds,$
where,\\ $$W_l \eqdef \abs{\phase{\left(\alpha \conj{a}_{1l}+\sqrt{1-\alpha^2} \conj{a}_{2l}\right) a_{1l}}-1}.$$
\begin{align*}
  &\prob{W_l > \frac{\sqrt{t}}{s} \middle| \abs{a_{1l}} = s} \\
	&= \prob{\abs{\phase{1+\frac{\sqrt{1-\alpha^2} \conj{a}_{2l}}{\alpha \conj{a}_{1l}}}-1} > \frac{\sqrt{t}}{s}
		\middle| \abs{a_{1l}} = s} \\
	&\stackrel{(\zeta_1)}{\leq} \prob{\frac{\sqrt{1-\alpha^2} \abs{a_{2l}}}{\alpha \abs{a_{1l}}} > \frac{c\sqrt{t}}{s}
		\middle| \abs{a_{1l}} = s} \\
	&\stackrel{(\zeta_2)}{\leq} \exp\left(1-\frac{c\alpha^2 t}{1-\alpha^2}\right),
\end{align*}
where $(\zeta_1)$ uses Lemma~\ref{lem:phase-absvalue} and $(\zeta_2)$, the fact that
$a_{2l}$ is a sub-gaussian random variable. This means:
{
\begin{align}
  \prob{U_l > t} &\leq \int_{\frac{\sqrt{t}}{2}}^{\infty} \exp\left(1-\frac{c\alpha^2 t}{1-\alpha^2}\right) p_{\abs{a_{1l}}}(s) ds \nonumber \\
   &\leq \exp\left(1-\frac{c\alpha^2 t}{1-\alpha^2}\right) \int_{\frac{\sqrt{t}}{2}}^{\infty} s e^{-\frac{s^2}{2}} ds \nonumber \\
  &\leq \exp\left(1 - \frac{ct}{1-\alpha^2}\right). \label{eqn:Ultailbound}
\end{align}}
Using this, we have the following bound on the expected value of $U_l$:
\begin{align*}
  \expec{U_l} = \int_0^{\infty} \prob{U_l > t} dt 
	&\leq c\left(1-\alpha^2\right).
\end{align*}
From \eqref{eqn:Ultailbound}, we see that $U_l$ is a subexponential random variable with parameter $c\left(1-\alpha^2\right)$.
Using Proposition 5.16 from \cite{Vershynin10}, we obtain:
\begin{align*}
  &\prob{\abs{\sum_{l=1}^m U_l - \expec{U_l}} > \delta m \left(1-\alpha^2\right)} \\
	&\leq 2 \exp\left(-\min\left(\frac{c\delta^2 m^2 \left(1-\alpha^2\right)^2}{\left(1-\alpha^2\right)^2m},
		\frac{c\delta m \left(1-\alpha^2\right)}{1-\alpha^2}\right)\right) \\
	&\leq 2 \exp\left(-c\delta^2 m\right) \leq \frac{\eta}{4}.
\end{align*}
\vspace{-1cm}

\end{proof}

\begin{lemma}\label{lem:zbound}
Assume the hypothesis of Theorem \ref{thm:convergence} and let $\vecfont{\xplus}$ be as defined in \eqref{eq:xupdate}. Then, for every unit vector $\vecfont{z}$ s.t. $\iprod{\vecfont{z}}{\xo}=0$, the following holds
(w.p. $\geq 1-\frac{\eta}{4}e^{-n}$):  $|\iprod{\vecfont{z}}{\vecfont{\xplus}}|\leq \frac{5}{9} \distop{\xo}{\vecfont{x}}$.
\end{lemma}
\begin{proof}
Fix $\vecfont{z}$ such that $\iprod{\vecfont{z}}{\xo}=0$.
Since the distribution of $A$ is rotationally invariant, wlog we can assume that:
a)	$\vecfont{x^*} = \vecfont{e_1}$,
b)	$\vecfont{x} = \alpha \vecfont{e_1} + \sqrt{1-\alpha^2} \vecfont{e_2}$ where $\alpha\in \reals$ and $\alpha \geq 0$ and
c)	$\vecfont{z} = \beta \vecfont{e_2} + \sqrt{1-\abs{\beta}^2} \vecfont{e_3}$ for some $\beta \in \complex$.
Note that we first prove the lemma for a {\em fixed} $\z$ and then use union bound.
For a fixed $z$, we have:
\begin{equation}\label{eq:t3}\left|\iprod{\vecfont{z}}{\vecfont{\xplus}}\right| \leq \left|\beta\right| |\iprod{\vecfont{e_2}}{\vecfont{\xplus}}| + \sqrt{1-\left|\beta\right|^2} |\iprod{\vecfont{e_3}}{\vecfont{\xplus}}|.\end{equation}
Now,  
\begin{align}&\abs{\vecfont{e_2}^T \vecfont{\xplus}}
= \abs{\vecfont{e_2}^T \left(\mat{A}\mat{A}^T\right)^{-1}\mat{A}\left(\mat{D}-\mat{I}\right)\mat{A}^T\vecfont{e_1}} \nonumber\\
	&\leq \frac{1}{2m}\abs{\vecfont{e_2}^T \left(\left(\frac{1}{2m}\mat{A}\mat{A}^T\right)^{-1} - \mat{I}\right)\mat{A}\left(\mat{D}-\mat{I}\right)\mat{A}^T\vecfont{e_1}} \nonumber \\
	& \quad  + \frac{1}{2m}\abs{\vecfont{e_2}^T \mat{A}\left(\mat{D}-\mat{I}\right)\mat{A}^T\vecfont{e_1}} \nonumber \\
	&\leq \frac{1}{2m}\twonorm{\left(\frac{1}{2m}\mat{A}\mat{A}^T\right)^{-1} - \mat{I}}\twonorm{\mat{A}}\twonorm{\left(\mat{D}-\mat{I}\right)\mat{A}^T\vecfont{e_1}} \nonumber \\
	&\quad  + \frac{1}{2m}\abs{\vecfont{e_2}^T \mat{A}\left(\mat{D}-\mat{I}\right)\mat{A}^T\vecfont{e_1}}, \nonumber\\
&\hspace{-0.3cm}\leq \frac{4c}{\sqrt{\widehat{c}}}\distop{\vecfont{x^t}}{\xo}+\frac{1}{2m}\abs{\vecfont{e_2}^T \mat{A}\left(\mat{D}-\mat{I}\right)\mat{A}^T\vecfont{e_1}}, \label{eqn:e2Txplus}   
\end{align}
where the last step uses Lemma~\ref{lem:xo_comp}.
Similarly,  
\begin{align}
\abs{\vecfont{e_3}^T \vecfont{\xplus}}
&\leq \frac{4c}{\sqrt{\widehat{c}}}\distop{\vecfont{x^t}}{\xo} \nonumber \\
&\;\;+\frac{1}{2m}\abs{\vecfont{e_3}^T \mat{A}\left(\mat{D}-\mat{I}\right)\mat{A}^T\vecfont{e_1}}.   
\label{eqn:e3Txplus}
\end{align}
{
Using \eqref{eq:t3}, \eqref{eqn:e2Txplus}, \eqref{eqn:e3Txplus} along with Lemmas~\ref{lem:e2err} and \ref{lem:e3err}, we see that for a fixed $z$, we have:
\begin{align}
\abs{\iprod{\vecfont{z}}{\vecfont{\xplus}}} \leq \frac{51}{100} \distop{\xo}{\vecfont{x}}, \label{eqn:dist-bound-fixedz}
\end{align}
with probability greater than $1 - \frac{\eta}{10} \exp(-cn)$.

So far we have proved the result only for a fixed vector $z$. We now use a covering and union bound argument to extend this result for {\em every} $z$ that is orthogonal to $\xo$.

\textbf{Union bound argument}:
Construct an $\epsilon$-net $S$ for unit vectors in the $(n-1)$-dimensional space that is orthogonal to $\xo$. Using standard results (see e.g., Chap. 13 of \cite{LorentzvGM1996}), we know that the size of $S$ can be chosen to be $\left(\frac{1}{\epsilon}\right)^{\order{n}}$. We choose $\epsilon = 1/100$, and hence the size of $S$ is $\exp\left({cn}\right)$, for some fixed constant $c$.
Applying \eqref{eqn:dist-bound-fixedz} for every $\z \in S$, and taking a union bound, we obtain:
\begin{align}
\abs{\iprod{\vecfont{z}}{\vecfont{\xplus}}} \leq \frac{51}{100} \distop{\xo}{\vecfont{x}} \; \forall \; \z \in S,
\label{eqn:dist-bound-epsnetz}
\end{align}
with probability greater than $1 - \frac{\eta}{10}\exp(-n)$.

Now choose a unit vector $\zhat$ that is orthogonal to $\xo$ (but is not necessarily in S), that maximizes $\abs{\iprod{\zhat}{\xplus}}$. In other words, $\zhat$ is such that
\begin{align}
\zhat \in \argmax_{\stackrel{\z \perp \xo}{\twonorm{\z}=1}} \abs{\iprod{\z}{\xo}}. \label{eqn:opt-iprod}
\end{align}

Since $S$ is a $\frac{1}{100}$-net of the orthogonal space to $\xo$, we know that there is a $\z \in S$ such that $\twonorm{\z-\zhat} < \frac{1}{100}$. So, we have:
\begin{align*}
\abs{\iprod{\zhat}{\xo}}
&\leq \abs{\iprod{\z}{\xo}} + \abs{\iprod{\zhat - \z}{\xo}} \\
&\stackrel{(\zeta_1)}{\leq} \frac{51}{100} \distop{\xo}{\vecfont{x}} + \frac{1}{100} \abs{\iprod{\frac{\zhat - \z}{\twonorm{\zhat-\z}}}{\xo}} \\
&\stackrel{(\zeta_2)}{\leq} \frac{51}{100} \distop{\xo}{\vecfont{x}} + \frac{1}{100} \abs{\iprod{\zhat}{\xo}},
\end{align*}
where $(\zeta_1)$ follows from \eqref{eqn:dist-bound-epsnetz} and $(\zeta_2)$ follows from \eqref{eqn:opt-iprod}.
This means that
\begin{align*}
\abs{\iprod{\zhat}{\xo}} \leq \frac{51}{99} \distop{\xo}{\vecfont{x}}.
\end{align*}
Recalling the choice of $\zhat$ from \eqref{eqn:opt-iprod} finishes the proof.
}
\end{proof}

\begin{lemma}\label{lem:e2err}
  Assume the hypothesis of Theorem~\ref{thm:convergence} and the notation therein. Then,
\begin{align*}
  \abs{\vecfont{e_2}^T \mat{A}\left(\mat{D}-\mat{I}\right)\mat{A}^T\vecfont{e_1}} \leq \frac{100}{99} m \sqrt{1-\alpha^2},
\end{align*}
with probability greater than $1-\frac{\eta}{10}e^{-n}$.
\end{lemma}
\begin{proof}
We have:
\begin{align*}
  &\vecfont{e_2}^T \mat{A}\left(\mat{D}-\mat{I}\right)\mat{A}^T\vecfont{e_1} \\
	&= \sum_{l=1}^m \conj{a}_{1l} a_{2l} \left(\phase{\left(\alpha \conj{a}_{1l}+\sqrt{1-\alpha^2} \conj{a}_{2l}\right) a_{1l}}-1\right) \\
	&= \sum_{l=1}^m \abs{a_{1l}} a'_{2l} \left(\phase{\alpha \abs{a_{1l}}+\sqrt{1-\alpha^2} \conj{a'_{2l}}}-1\right),
\end{align*}
where $a'_{2l} \eqdef a_{2l}\phase{\conj{a}_{1l}}$ is identically distributed to $a_{2l}$ and is independent of
$\abs{a_{1l}}$. Define the random variable $U_l$ as:
\begin{align*}
  U_l \eqdef \abs{a_{1l}} a'_{2l} \left(\phase{1+\frac{\sqrt{1-\alpha^2} \conj{a'_{2l}}}{\alpha \abs{a_{1l}}}}-1\right).
\end{align*}
Similar to Lemma \ref{lem:xo_comp}, we will calculate $\prob{U_l > t}$ to show that $U_l$ is subexponential and use it to derive concentration
bounds. However, using the above estimate to bound $\expec{U_l}$ will result in a weak bound that we will not be able to use.
Lemma \ref{lem:e2err-expected-value-bound} bounds $\expec{U_l}$ using a different technique carefully.
\begin{align*}
  &\prob{\abs{U_l} > t} \\
  &\leq \prob{ \abs{a_{1l}} \abs{a'_{2l}} \frac{c\sqrt{1-\alpha^2} \abs{a'_{2l}}}{\alpha \abs{a_{1l}}} > t} \\
	&= \prob{ \abs{a'_{2l}}^2  > \frac{c\alpha t}{\sqrt{1-\alpha^2}}}
	\leq \exp\left(1-\frac{c\alpha t}{\sqrt{1-\alpha^2}}\right),
\end{align*}
where the last step follows from the fact that $a'_{2l}$ is a subgaussian random variable and hence $\abs{a'_{2l}}^2$ is
a subexponential random variable. Using Proposition 5.16 from \cite{Vershynin10}, we obtain:
\begin{align*}
  &\prob{\abs{\sum_{l=1}^m U_l - \expec{U_l}} > \delta m \sqrt{1-\alpha^2}}\\
	&\leq 2 \exp\left(-\min\left(\frac{c\delta^2 m^2 \left(1-\alpha^2\right)}{\left(1-\alpha^2\right)m},
		\frac{c\delta m \sqrt{1-\alpha^2}}{\sqrt{1-\alpha^2}}\right)\right) \\
	&\leq 2 \exp\left(-c\delta^2 m\right).
\end{align*}
\rededits{
Choosing $\delta = \frac{1}{99}$ and using Lemma \ref{lem:e2err-expected-value-bound}, we obtain:
\begin{align*}
  \abs{\vecfont{e_2}^T \mat{A}\left(\mat{D}-\mat{I}\right)\mat{A}^T\vecfont{e_1}}
	&= \abs{\sum_{l=1}^m U_l}
	\leq \frac{100}{99} m \sqrt{1-\alpha^2},
\end{align*}
with probability greater than $1 - \frac{\eta}{10}\exp(-n)$.} This proves the lemma.
\end{proof}

\begin{proof}[Proof of Lemma~\ref{lem:e2err-expected-value-bound}]
Let $w_2 = \abs{w_2} e^{i\theta}$. Then $\abs{w_1}, \abs{w_2}$ and $\theta$ are all independent random variables. $\theta$ is a uniform random
variable over $[-\pi,\pi]$ and $\abs{w_1}$ and $\abs{w_2}$ are identically distributed with probability distribution function,
  $p(x) = x \exp\left(-\frac{x^2}{2}\right) \mathds{1}_{\{x \geq 0\}}$.
We have:
\begin{align*}
  &\expec{U} =
  \mathbb{E}\left[ \abs{w_1} \abs{w_2} \right. \\
	& \left. \mathbb{E}\left[e^{i\theta} \left(\phase{1+\frac{\sqrt{1-\alpha^2} \abs{w_2}e^{-i\theta}}{\alpha \abs{w_1}}}-1\right) \middle| \abs{w_1},\right.\right. \\
	&\qquad \qquad\qquad \qquad \qquad \qquad \qquad \qquad \qquad \quad \left. \abs{w_2} \right].
\end{align*}
Let $\beta \eqdef \frac{\sqrt{1-\alpha^2} \abs{w_2}}{\alpha \abs{w_1}}$. We will first calculate $\expec{e^{i\theta} \phase{1+\beta e^{-i\theta}}\middle| \abs{w_1}, \abs{w_2}}$.
Note that the above expectation is taken only over the randomness in $\theta$.
For simplicity of notation, we will drop the conditioning variables, and calculate the above expectation in terms of $\beta$ as
\begin{align*}
  e^{i\theta} \phase{1+\beta e^{-i\theta}}
	&= \frac{\cos \theta + \beta + i \sin \theta}{\left(1+\beta^2+2\beta\cos\theta\right)^{\frac{1}{2}}}.
\end{align*}
We will first calculate the imaginary part of the above expectation:
\begin{align}
  &\imag{\expec{e^{i\theta} \phase{1+\beta e^{-i\theta}}}} \nonumber \\
  &\qquad = \expec{\frac{\sin \theta}{\left(1+\beta^2+2\beta\cos\theta\right)^{\frac{1}{2}}}} = 0, \label{eqn:e2err-expec-imag}
\end{align}
since we are taking the expectation of an odd function. Focusing on the real part, we let:
\begin{align*}
  F(\beta) &\eqdef \expec{\frac{\cos \theta + \beta}{\left(1+\beta^2+2\beta\cos\theta\right)^{\frac{1}{2}}}} \\
	&= \frac{1}{2\pi} \int_{-\pi}^\pi \frac{\cos \theta + \beta}{\left(1+\beta^2+2\beta\cos\theta\right)^{\frac{1}{2}}} d\theta.
\end{align*}
Note that $F(\beta):\reals \rightarrow \reals$ and $F(0)=0$. We will show that there is a small absolute numerical constant $\gamma$ (depending on $\delta$) such that:
\begin{align}\label{eqn:e2err-expec-smallbetabound}
  0 < \beta < \gamma \Rightarrow \abs{F(\beta)} \leq (\frac{1}{2}+\delta) \beta.
\end{align}
We show this by calculating $F'(0)$ and using the continuity of $F'(\beta)$ at $\beta = 0$.
We first calculate $F'(\beta)$ as follows:
\begin{align*}
  F'(\beta) &= \frac{1}{2\pi} \int_{-\pi}^\pi \frac{1}{\left(1+\beta^2+2\beta\cos\theta\right)^{\frac{1}{2}}} \\
  & \qquad \quad -  \frac{\left(\cos \theta + \beta\right)\left(\beta+\cos \theta\right)}{\left(1+\beta^2+2\beta\cos\theta\right)^{\frac{3}{2}}} d\theta \\
	&= \frac{1}{2\pi} \int_{-\pi}^\pi \frac{\sin^2 \theta}{\left(1+\beta^2+2\beta\cos\theta\right)^{\frac{3}{2}}} d\theta
\end{align*}
From the above, we see that $F'(0)=\frac{1}{2}$ and \eqref{eqn:e2err-expec-smallbetabound} then follows from the continuity of $F'(\beta)$ at $\beta=0$.
Getting back to the expected value of $U$, we have:
\begin{align}
  &\abs{\expec{U}} \nonumber \\
	&\leq \left|\expec{\abs{w_1} \abs{w_2} F\left(\frac{\sqrt{1-\alpha^2} \abs{w_2}}{\alpha \abs{w_1}}\right) \mathds{1}_{\left\{\frac{\sqrt{1-\alpha^2} \abs{w_2}}{\alpha \abs{w_1}}< \gamma\right\}}} \right|\nonumber\\
	&+ \left|\expec{\abs{w_1} \abs{w_2} F\left(\frac{\sqrt{1-\alpha^2} \abs{w_2}}{\alpha \abs{w_1}}\right) \mathds{1}_{\left\{\frac{\sqrt{1-\alpha^2} \abs{w_2}}{\alpha \abs{w_1}}\geq \gamma\right\}}}\right| \nonumber\\
	&\stackrel{(\zeta_1)}{\leq} \left(\frac{1}{2}+\delta\right)\expec{\abs{w_1} \abs{w_2} \frac{\sqrt{1-\alpha^2} \abs{w_2}}{\alpha \abs{w_1}}} \nonumber \\
	&\qquad \qquad + \expec{\abs{w_1} \abs{w_2} \mathds{1}_{\left\{\frac{\sqrt{1-\alpha^2} \abs{w_2}}{\alpha \abs{w_1}}\geq \gamma\right\}}},\nonumber\\
	&\stackrel{(\zeta_2)}{=} \left(1+2\delta\right)\left(\frac{\sqrt{1-\alpha^2}}{\alpha}\right) \nonumber \\
	&\qquad \qquad + \expec{\abs{w_1} \abs{w_2} \mathds{1}_{\left\{\frac{\sqrt{1-\alpha^2} \abs{w_2}}{\alpha \abs{w_1}}\geq \gamma\right\}}},\label{eqn:e2err-expec-betabound}
\end{align}
where $(\zeta_1)$ follows from \eqref{eqn:e2err-expec-smallbetabound} and the fact that $\abs{F(\beta)} \leq 1$ for every $\beta$ and
$(\zeta_2)$ follows from the fact that $\expec{\abs{z_2}^2}=2$.
We will now bound the second term in the above inequality. We start with the following integral:
\begin{align}
  &\int_t^{\infty} s^2 e^{-\frac{s^2}{2}} ds = -\int_t^{\infty} s d\left(e^{-\frac{s^2}{2}}\right) \nonumber \\
	&= te^{-\frac{t^2}{2}} + \int_t^{\infty} e^{-\frac{s^2}{2}} ds
	\leq (t+e) e^{-\frac{t^2}{c}},\label{eqn:e2err-expec-intbound}
\end{align}
where $c$ is some constant. The last step follows from standard bounds on the tail probabilities of gaussian random variables.
We now bound the second term of \eqref{eqn:e2err-expec-betabound} as follows:
\begin{align*}
  &\expec{\abs{w_1} \abs{w_2} \mathds{1}_{\left\{\frac{\sqrt{1-\alpha^2} \abs{w_2}}{\alpha \abs{w_1}}\geq \gamma\right\}}} \\
	&= \int_0^{\infty} t^2 e^{-\frac{t^2}{2}} \int_{\frac{\alpha t}{\sqrt{1-\alpha^2}}}^{\infty} s^2 e^{-\frac{s^2}{2}} ds dt \\
	&\stackrel{(\zeta_1)}{\leq} \int_0^{\infty} t^2 e^{-\frac{t^2}{2}} \left(\frac{\alpha t}{\sqrt{1-\alpha^2}}+e\right) e^{-\frac{\alpha^2 t^2}{c\left(1-\alpha^2\right)}} dt \\
	&\leq \int_0^{\infty} \left(\frac{\alpha t^3}{\sqrt{1-\alpha^2}}+et^2\right) e^{-\frac{t^2}{c\left(1-\alpha^2\right)}} dt \\
	&= \frac{\alpha}{\sqrt{1-\alpha^2}}\int_0^{\infty} t^3 e^{-\frac{t^2}{c\left(1-\alpha^2\right)}} dt + e\int_0^{\infty} t^2 e^{-\frac{t^2}{c\left(1-\alpha^2\right)}} dt \\
	&\stackrel{(\zeta_2)}{\leq} c \left(1-\alpha^2\right)^{\frac{3}{2}}
	\stackrel{(\zeta_3)}{\leq} \delta \sqrt{1-\alpha^2}
\end{align*}
where $(\zeta_1)$ follows from \eqref{eqn:e2err-expec-intbound}, $(\zeta_2)$ follows from the formulae for second and third absolute moments of gaussian random variables
and $(\zeta_3)$ follows from the fact that $1-\alpha^2 < \delta$. Plugging the above inequality in \eqref{eqn:e2err-expec-betabound}, we obtain:
\begin{align*}
  \abs{\expec{U}} &\leq \left(1+2\delta\right)\left(\frac{\sqrt{1-\alpha^2}}{\alpha}\right) + \delta \sqrt{1-\alpha^2} \\
	&\leq \left(1+4\delta\right)\sqrt{1-\alpha^2},
\end{align*}
where we used the fact that $\alpha \geq 1 - \frac{\delta}{2}$. This proves the lemma.
\end{proof}

\begin{lemma}\label{lem:e3err}
  Assume the hypothesis of Theorem~\ref{thm:convergence} and the notation therein. Then,
\begin{align*}
  \abs{\vecfont{e_3}^T \mat{A}\left(\mat{D}-\mat{I}\right)\mat{A}^T\vecfont{e_1}} \leq \frac{1}{100} m \sqrt{1-\alpha^2},
\end{align*}
with probability greater than $1-\frac{\eta}{10}e^{-n}$.
\end{lemma}
\begin{proof}
The proof of this lemma is very similar to that of Lemma \ref{lem:e2err}.
We have:
\begin{align*}
  &\vecfont{e_3}^T \mat{A}\left(\mat{D}-\mat{I}\right)\mat{A}^T\vecfont{e_1} \\
	&= \sum_{l=1}^m \conj{a}_{1l} a_{3l} \left(\phase{\left(\alpha \conj{a}_{1l}+\conj{a}_{2l}\sqrt{1-\alpha^2} \conj{a}_{3l}\right) a_{1l}}-1\right) \\
	&= \sum_{l=1}^m \abs{a_{1l}} a'_{3l} \left(\phase{\alpha \abs{a_{1l}}+\conj{a'_{2l}}\sqrt{1-\alpha^2} }-1\right),
\end{align*}
where $a'_{3l} \eqdef a_{3l}\phase{\conj{a}_{1l}}$ is identically distributed to $a_{3l}$ and is independent of
$\abs{a_{1l}}$ and $a'_{2l}$. Define the random variable $U_l$ as:
\begin{align*}
  U_l \eqdef \abs{a_{1l}} a'_{3l} \left(\phase{1+\frac{\conj{a'_{2l}}\sqrt{1-\alpha^2} }{\alpha \abs{a_{1l}}}}-1\right).
\end{align*}
Since $a'_{3l}$ has mean zero and is independent of everything else, we have
  $\expec{U_l} = 0$.
Similar to Lemma \ref{lem:e2err}, we will calculate $\prob{U_l > t}$ to show that $U_l$ is subexponential and use it to derive concentration
bounds.
\begin{align*}
  &\prob{\abs{U_l} > t}
  \leq \prob{ \abs{a_{1l}} \abs{a'_{3l}} \frac{c\sqrt{1-\alpha^2} \abs{a'_{2l}}}{\alpha \abs{a_{1l}}} > t} \\
	&= \prob{ \abs{a'_{2l}a'_{3l}}  > \frac{c\alpha t}{\sqrt{1-\alpha^2}}}
	\leq \exp\left(1-\frac{c\alpha t}{\sqrt{1-\alpha^2}}\right),
\end{align*}
where the last step follows from the fact that $a'_{2l}$ and $a'_{3l}$ are independent subgaussian random variables and hence $\abs{a'_{2l}a'_{3l}}$ is
a subexponential random variable. Using Proposition 5.16 from \cite{Vershynin10}, we obtain:
\begin{align*}
  &\prob{\abs{\sum_{l=1}^m U_l - \expec{U_l}} > \delta m \sqrt{1-\alpha^2}} \\
	&\leq 2 \exp\left(-\min\left(\frac{c\delta^2 m^2 \left(1-\alpha^2\right)}{\left(1-\alpha^2\right)m},
		\frac{c\delta m \sqrt{1-\alpha^2}}{\sqrt{1-\alpha^2}}\right)\right) \\
	&\leq 2 \exp\left(-c\delta^2 m\right).
\end{align*}
\rededits{
Choosing $\delta = \frac{1}{100}$, we have:
\begin{align*}
  \abs{\vecfont{e_3}^T \mat{A}\left(\mat{D}-\mat{I}\right)\mat{A}^T\vecfont{e_1}}
	&= \abs{\sum_{l=1}^m U_l}
	\leq \frac{1}{100} m \sqrt{1-\alpha^2},
\end{align*}
with probability greater than $1 - \frac{\eta}{10}\exp(-n)$.} This proves the lemma.
\end{proof}

\begin{lemma}\label{lem:phase-absvalue}
For every $w \in \complex$, we have:
\begin{align*}
  \abs{\phase{1+w}-1} \leq 2 \abs{w}.
\end{align*}
\end{lemma}
\begin{proof}
The proof is straight forward:
\begin{align*}
  \abs{\phase{1+w}-1} &\leq \abs{\phase{1+w}-(1+w)} + \abs{w} \\
	&= \abs{1-\abs{1+w}} + \abs{w}
	\leq 2 \abs{w}.
\end{align*}
\vspace{-1.1cm}

\end{proof}

\section{Proofs for Section \ref{sec:sparse}}\label{app:sparse}
\begin{proof}[Proof of Lemma \ref{lem:sparseretrieval-supportrecovery}]
  For every $j \in [n]$ and $i \in [m]$, consider the random variable $Z_{ij} \eqdef \left|a_{ij}y_i\right|$. We have the following:
 \begin{itemize}
  \item	if $j\in S$, then
  \begin{align*}
    \expec{Z_{ij}} &= \frac{2}{\pi}\left(\sqrt{1 - \left(x^*_j\right)^2} + x^*_j \arcsin x^*_j\right) \\
		&\geq \frac{2}{\pi}\left(1 - \frac{5}{6}\left(x^*_j\right)^2 - \frac{1}{6}\left(x^*_j\right)^4 \right. \\
		&\qquad \qquad \left.+ x^*_j \left(x^*_j + \frac{1}{6}\left(x^*_j\right)^3\right) \right) \\
		&\geq \frac{2}{\pi} + \frac{1}{6} \left(x_{\textrm{min}}^*\right)^2,
  \end{align*}
 where the first step follows from Corollary 3.1 in \cite{LiW09} and the second step follows from the Taylor series expansions of $\sqrt{1-x^2}$ and $\arcsin(x)$,
  \item	if $j\notin S$, then $\expec{Z_{ij}} = \expec{\left|a_{ij}\right|} \expec{\left|y_{i}\right|}= \frac{2}{\pi}$ and finally,
  \item	for every $j\in[n]$, $Z_{ij}$ is a sub-exponential random variable with parameter $c = O(1)$ (since it is a product of two standard normal random variables).
 \end{itemize}
Using the hypothesis of the theorem about $m$, we have:
\begin{itemize}
  \item	for any $j\in S$, $\prob{\frac{1}{m}\sum_{i=1}^m Z_{ij} - \left(\frac{2}{\pi} +
	\frac{1}{12} \left(x_{\textrm{min}}^*\right)^2\right) < 0} \leq \exp\left(-c \left(x_{\textrm{min}}^*\right)^4 m\right) \leq \delta n^{-c}$, and
  \item	for any $j\notin S$, $\prob{\frac{1}{m}\sum_{i=1}^m Z_{ij} - \left(\frac{2}{\pi} +
	\frac{1}{12} \left(x_{\textrm{min}}^*\right)^2\right) > 0} \leq \exp\left(-c \left(x_{\textrm{min}}^*\right)^4 m\right) \leq \delta n^{-c}$.
\end{itemize}
Applying a union bound to the above, we see that with probability greater than $1-\delta$, there is a separation in the values of $\frac{1}{m} \sum_{i=1}^m Z_{ij}$
for $j\in S$ and $j\notin S$. This proves the theorem.
\end{proof}


%
%

\end{document}